\documentclass[letterpaper]{article} 
\usepackage{aaai23}  
\usepackage{times}  
\usepackage{helvet}  
\usepackage{courier}  
\usepackage[hyphens]{url}  
\usepackage{graphicx} 
\usepackage{natbib}  
\usepackage{caption} 
\usepackage{algorithm}
\usepackage{algorithmic}
\usepackage{makecell}
\usepackage{booktabs}
\usepackage{ulem}
\usepackage{multirow}

\usepackage{amsthm}

\usepackage{newfloat}
\usepackage{listings}
\DeclareCaptionStyle{ruled}{labelfont=normalfont,labelsep=colon,strut=off} 
\floatstyle{ruled}
\newfloat{listing}{tb}{lst}{}
\floatname{listing}{Listing}
\usepackage{color}  
\usepackage{ulem} 
\usepackage{amstext}
\usepackage{amsmath, bm}
\usepackage{amsfonts}

\newcounter{ass_counter}
\newcounter{thm_counter}
\newcounter{coro_counter}
\newcounter{lem_counter}
\newcounter{remark_counter}

\newtheorem{theorem}[thm_counter]{Theorem}
\newtheorem{lemma}[lem_counter]{Lemma}
\newtheorem{corollary}[coro_counter]{Corollary}
\newtheorem{assumption}[ass_counter]{Assumption}
\newtheorem{remark}[remark_counter]{Remark}
\newtheorem{definition}{Definition}

\title{Differentially Private Learning with Per-Sample Adaptive Clipping}
\author{
    Tianyu Xia\textsuperscript{\rm 4},
    Shuheng Shen\textsuperscript{\rm 5},
    Su Yao\textsuperscript{\rm 1,*},
    Xinyi Fu\textsuperscript{\rm 5},
    Ke Xu\textsuperscript{\rm 2,3,*},
    Xiaolong Xu\textsuperscript{\rm 5},
    Xing Fu\textsuperscript{\rm 5}
}
\affiliations{
    \textsuperscript{\rm 1}Beijing National Research Center for Information Science and Technology (BNRist), Tsinghua University\\
    \textsuperscript{\rm 2}Department of Computer Science \& Technology, Tsinghua University
    \textsuperscript{\rm 3}Zhongguancun Laboratory, Beijing\\
    \textsuperscript{\rm 4}School of Software \& Microelectronics, Peking University
    \textsuperscript{\rm 5}Tiansuan Lab, Ant Group\\
    xiatainyu@stu.pku.edu.cn, \{yaosu, xuke\}@tsinghua.edu.cn,
    \{shuheng.ssh, fxy122992, yiyin.xxl, zicai.fx \}@antgroup.com
}

\usepackage{bibentry}

\begin{document}

\maketitle

\begin{abstract}
Privacy in AI remains a topic that draws attention from researchers and the general public in recent years. As one way to implement privacy-preserving AI, differentially private learning is a framework that enables AI models to use differential privacy (DP). To achieve DP in the learning process, existing algorithms typically limit the magnitude of gradients with a constant clipping, which requires carefully tuned due to its significant impact on model performance. As a solution to this issue, latest works NSGD and Auto-S innovatively propose to use normalization instead of clipping to avoid hyperparameter tuning. However, normalization-based approaches like NSGD and Auto-S rely on a monotonic weight function, which imposes excessive weight on small gradient samples and introduces extra deviation to the update. In this paper, we propose a Differentially Private Per-Sample Adaptive Clipping (DP-PSAC) algorithm based on a non-monotonic adaptive weight function, which guarantees privacy without the typical hyperparameter tuning process of using a constant clipping while significantly reducing the deviation between the update and true batch-averaged gradient. We provide a rigorous theoretical convergence analysis and show that with convergence rate at the same order, the proposed algorithm achieves a lower non-vanishing bound, which is maintained over training iterations, compared with NSGD/Auto-S.  In addition, through extensive experimental evaluation, we show that DP-PSAC outperforms or matches the state-of-the-art methods on multiple main-stream vision and language tasks.

\end{abstract}

\section{Introduction}

Machine learning has substantially benefited from deep learning research and implementation. Unfortunately, the success of deep neural networks depends on a substantial amount of high-quality data, much of which typically contain sensitive personal data, making data-driven deep models vulnerable to privacy leaks~\cite{zhu2019deep}. DP~\cite{dwork2014algorithmic} formally defines the influence of an individual sample on the final result and provides rigorous theoretical guarantees. Differentially Private stochastic gradient descent (DP-SGD)~\cite{abadi2016deep}, which first clips each stochastic gradient $g_t$ with a predetermined constant $C$ to constrain the privacy sensitivity and then adds Gaussian noise to the gradients to perturb the result, is a popularly used algorithm to defend deep learning models from differential attacks. 
Specifically, the iteration of DP-SGD at $x_t$ is:
{\small
$$ x_{t+1} = x_{t} - \frac{\eta_t}{|B_t|} \left ( \sum_{i \in B_t} g_{t,i} \min\left(\frac{C}{\|g_{t,i}\|}, 1\right) + \mathcal{N}(0, C^2\sigma^2) \right),$$
}where $\eta_t$ is the learning rate, $B_t$ is the random batch and $\sigma$ is the standard deviation of  Gaussian noise. Despite its considerable success, DP-SGD with constant clipping suffers from the following issues: 
\begin{itemize}
    \item 
    The performance of the final model, as~\citet{kurakin2022toward} noted, will be significantly impacted by an incorrect $C$. It is really challenging to 
    tune $C$.
    \item 
    The search for $C$ itself incurs a extra privacy budget~\cite{papernot2021hyperparameter}.
\end{itemize}
In order to obtain an optimal clipping threshold to achieve higher model accuracy, \citet{andrew2021differentially} estimated the optimal clipping threshold through gradient quantiles, but this introduces a bigger hyperparameter search space and a large amount of extra computation. By using a public dataset sampled from the private dataset or partial statistics of the private dataset, \citet{zhang2018differentially} estimated the optimal clipping threshold during the learning process, but this may lead to new privacy leaking problems.

To solve the aforementioned problems, two concurrent research~\cite{bu2022automatic,yang2022normalized} proposed to replace the clipping threshold with automatic clipping/normalizing, i.e. $\tilde{g}=g/(\|g\|+r)$, which can constrain the privacy sensitivity by normalizing all per-sample gradients to the same magnitude, but it actually assigns different weights to samples with various gradient norms. Consequently, the batch gradient becomes a weighted average of the per-sample gradients and the weighted gain is $1/(\| g \| + r)$, meaning smaller gradients are given larger weight. As shown in Figure~\ref{figure:fig2}, these techniques will increase the sample's weighted gain by up to $1/r$ times when its gradient norm moves toward 0, where $r$ is often set to 0.1 or a smaller value~\cite{bu2022automatic}. Unfortunately, as illustrated in Figure~\ref{figure:fig1}, we observe that in the iterative process, small gradient samples frequently have a tendency to be practically orthogonal or even opposite to the true batch-averaged gradient. This means that the contribution of small gradient samples to the true batch gradient is negligible. Thus, giving small gradient samples large weight results in an overwhelming deviation between the automatically clipped batch gradient and the actual batch gradient.

Intuitively, we hope that samples with different magnitudes of gradient norm will receive similar order of weights to preserve the average of clipped gradients as close to the original batch-averaged gradient as possible. Based on this, we propose \textbf{D}ifferentially \textbf{P}rivate \textbf{P}er-\textbf{S}ample \textbf{A}daptive \textbf{C}lipping (DP-PSAC) algorithm, by adopting a non-monotonous adaptive weight function. We summarize our contributions as follows:
\begin{itemize}
    \item We propose a per-sample adaptive clipping algorithm, which is a new perspective and orthogonal to dynamic adaptive noise~\cite{du2021dynamic} and coordinate clipping methods~\cite{pichapati2019adaclip,asi2021private}, and prove that it can be as private as currently used privacy-preserving optimization algorithms.
    \item 
    We show how our algorithm converges in non-convex settings and provide a convergence error bound under DP. In addition, we demonstrate that DP-PSAC has a lower non-vanishing bound than Auto-S/NSGD.
    \item 
    We demonstrate the empirical superiority of the proposed algorithm through extensive experiments while obtaining new state-of-the-art performance of differentially private learning on several datasets. 
\end{itemize}

\begin{figure}[t]
\centering
\includegraphics[width=0.9\columnwidth]{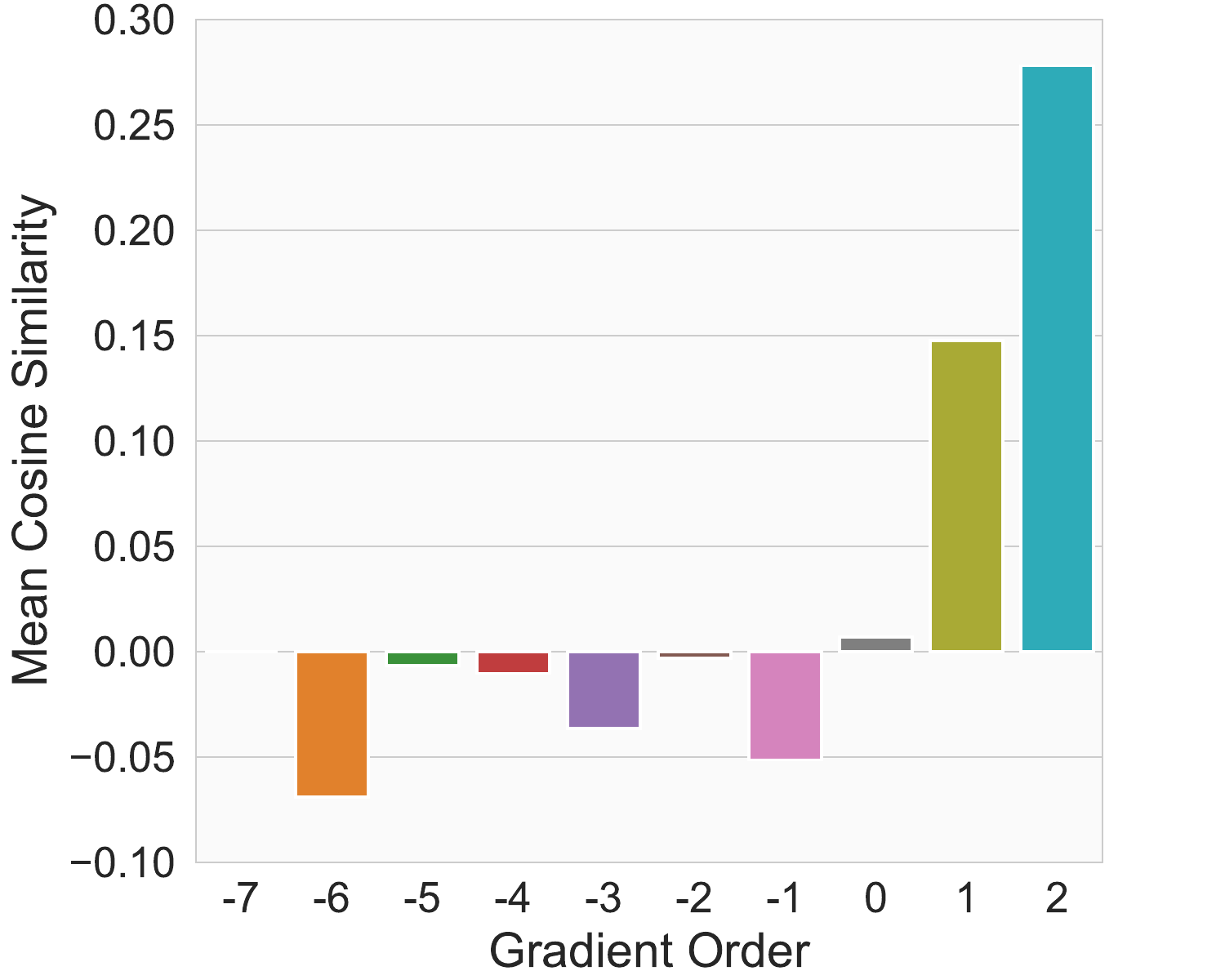} 
\caption{Average cosine similarity of single sample gradient and the batch-averaged gradient throughout training on MNIST dataset with DP-SGD under (3, $10^{-5}$)-DP.}
\label{figure:fig1}
\end{figure}

\begin{figure}[t]
\centering
\includegraphics[width=0.9\columnwidth]{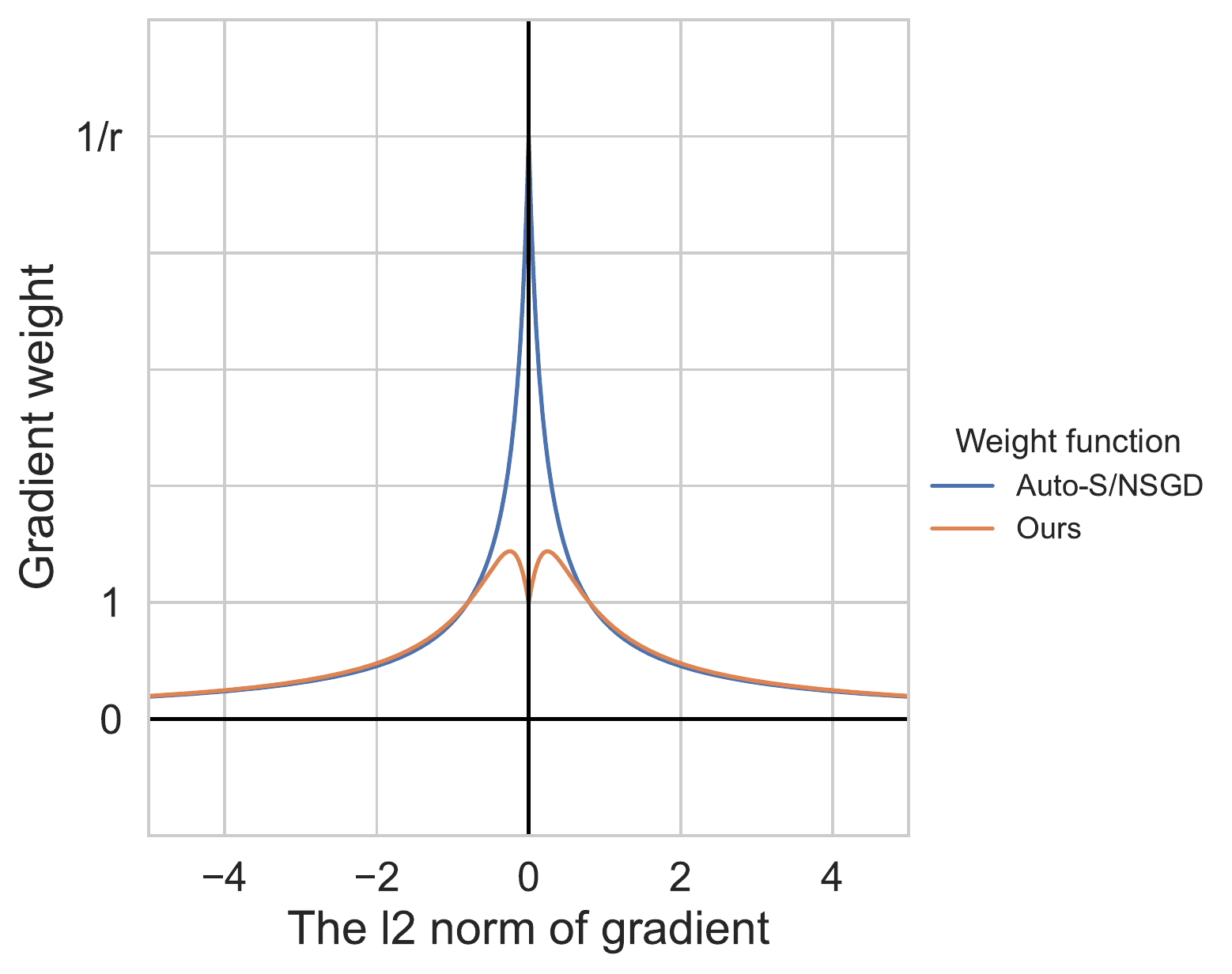} 
\caption{Gradient weight for calculating the batch-averaged gradient of our method and the Auto-S/NSGD method for different gradient norms.}
\label{figure:fig2}
\end{figure}

\section{Related work}
Deep learning based on gradient clipping and the Gaussian mechanism has become the most popular differentially private learning scheme.
Constant clipping was firstly adopted in~\cite{abadi2016deep} to equip SGD with privacy protection, called DP-SGD.
Subsequentially, it was well studied in a series of works~\cite{wang2017differentially,li2022private,wang2019differentially,kuru2022differentially,mangold2022differentially,bassily2021differentially,yu2021gradient,wang2022differentially,wu2021adaptive,esipova2022disparate} to apply DP to other optimization algorithms, such as DP-AdaGrad, DP-SVRG, and ApolySFW. From a theoretical perspective, \citet{zhang2020advances} and \citet{Zhang2020Why} analyzed the convergence of clipped SGD. From the perspective of application, DP-Lora~\cite{yu2022differentially} and RGP~\cite{yu2021LargeScale} enabled differential privacy learning for large-scale model fine-tuning through methods such as low-rank compression.

Nevertheless, it is shown that the optimal threshold is always changing during the optimization process~\cite{van2018three}.
Numerous studies are proposed to dynamically adjust the threshold in training in order to lessen the impact of a fixed threshold on the performance of DP-based algorithms.
Among them, \citet{andrew2021differentially} predicted the optimal clipping threshold using extra 
privacy budget during the optimization process.
\citet{du2021dynamic} proposed to dynamically decrease the clipping threshold and noise magnitude along with the iteration round $t$.
More fine-grained, some works~\cite{pichapati2019adaclip,asi2021private} proposed axis-level adaptive clipping and noise addition methods, giving different clipping thresholds and non-homogeneous noise to the gradient components on a different axis. 
Despite the great success of these algorithms, the initial threshold still needs to be manually set, and the final performance is sensitive to the initial threshold. 

To get rid of the dependence of differentially private learning on the clipping threshold, \citet{bu2022automatic} and \citet{yang2022normalized} concurrently proposed to constrain the gradient sensitivity with normalization, called Automatic Clipping (Auto-S) or Normalized SGD (NSGD). 
They showed that when normalizing all gradients to the same magnitude, the learning rate and the clipping hyperparameter can be coupled, thus only the one hyperparameter need to be tuned. 
However, this method suffers from a large deviation between their normalized batch-averaged gradient and the unnormalized one when some gradient norms in a batch are tiny. 
The proposed algorithm in this paper alleviates the above problem by reducing the size of deviation and achieves better theoretical and experimental results.

\section{Preliminary}

\subsection{Notations and Definitions}
Throughout the paper, we will let $\|\cdot\|$  denote the $\ell_2$ norm of a vector and $\langle \cdot, \cdot \rangle$ denote the inner product of two vectors. The gradient of $f(x)$ is represented by $\nabla f(x)$. The training dataset for the optimization problem is represented by $D$. The probability that event $z$ occurs is represented by $\rm Pr[z]$. A random variable's mathematical expectation is denoted by $\mathbb{E}(\cdot)$. We consider the following empirical risk minimization problem:
$$\mathop{min}\limits_{x \in R^d} f(x) := \frac{1}{|D|} \mathop{\sum}\limits_{\xi_i \in D} f(x, \xi_i),$$
where $f(x, \xi_i)$ is the loss function with respect to data point $\xi_i$. In addition, we use $x^*$ to indicate the optimal solution to the above problem.

DP~\cite{dwork2014algorithmic} provides a formal definition of individual privacy, with the intuition that the result of a random algorithm on a dataset should not be different too much with or without one data point:
\begin{definition}[$(\epsilon, \delta)$-DP] \label{def:dp}
A randomized mechanism $\mathcal{M}: \mathcal{D} \to \mathcal{R}$ offers $(\epsilon, \delta)$-differential privacy if for any two adjacent datasets $D, D' \in \mathcal{D}$ differing by a single data point and any $S \subset \mathcal{R}$ it satisfies that:
$${\rm Pr}[\mathcal{M}(D) \in S] \leq e^{\epsilon}{\rm Pr}[\mathcal{M}(D') \in S] + \delta.$$
\end{definition}
In deep learning training, $(\epsilon, \delta)$-DP is the most widely employed type of DP. It mainly relies on the Gaussian mechanism, which involves introducing Gaussian noise to gradients.
Its privacy budget can 
calculated by means of the moments accountant~\cite{abadi2016deep}, R{\'e}nyi-DP~\cite{mironov2017renyi} or $f$-DP~\cite{dong2019gaussian}.

\subsection{Assumptions}

In this paper, we formulate the following assumptions, all of which are common and basic in past works~\cite{ghadimi2013stochastic,bu2022automatic,yang2022normalized}.

\begin{assumption}[$(L_0,L_1)$-generalized smooth] 
We assume that $f(x)$ is $(L_0,L_1)$-generalized smooth, this is, for all $x,y \in \mathbb{R}^d$, there exist constants $L_0>0$ and $L_1\geq0$ such that $\|\nabla f(x)- \nabla f(y)\| \leq \left(L_0 + L_1 \|\nabla f(x)\| \right)\|x-y\|$.
\end{assumption}

\begin{assumption}[Bounded variance]
For all $x \in \mathbb{R}^d$, there exist constants $\tau_0 > 0$ and $0 \leq \tau_1 < 1$, such that $\|g(x, \xi_i) - \nabla f(x)\| \leq \tau_0 + \tau_1 \|\nabla f(x)\|$ with probability 1.
\end{assumption}

\subsection{Review: 
Normalized/Automatic DP Training}
The fundamental method of Normalized/Automatic differentially private training~\cite{bu2022automatic,yang2022normalized} is to limit the magnitude of each gradient by using normalization rather than clipping. Specifically, it normalizes all per-sample gradients to the same size:
\begin{equation}\nonumber
\displaystyle \widetilde{g} = {\rm Clip}(g) =  g/\|g\|.
\end{equation}

The algorithm called Auto-S/NSGD~\cite{bu2022automatic,yang2022normalized} 
jumps out of the original gradient clipping framework, so that the gradient clipping parameter and the learning rate are coupled:
\begin{eqnarray}
x_t - x_{t+1} &=& \frac{\eta_t}{|B_t|} \left(\sum_{i\in B_t} \frac{C g_{t,i}}{\| g_{t,i}\|} + \mathcal{N}(0, C^2 \sigma^2)\right) \nonumber \\
&=& \frac{\eta_t C}{|B_t|} \left(\sum_{i\in B_t} \frac{g_{t,i}}{\| g_{t,i}\|} + \mathcal{N}(0, \sigma^2)\right). \nonumber
\end{eqnarray}
As a result, it is unnecessary to tune the hyperparameter $C$. Additionally, a regularization term $r$ is added to the scaling factor to enhance training stability:
\begin{equation}
\displaystyle \widetilde{g} = {\rm Clip}(g) =  g/\left(\|g\|+r\right), \nonumber
\end{equation}
where $r$ is usually set to 0.1 or less~\cite{bu2022automatic}. 

On the one hand, Auto-S/NSGD outperforms standard clipping-based techniques on numerous vision and language tasks. On the other hand, it eliminates reliance on the clipping threshold and reduces the searching space for hyperparameters. The algorithm proposed in this paper is a refinement of Auto-S/NSGD.

\section{Motivation}

\subsection{Small Gradients Should not Get Huge Gains}

\subsubsection{The contribution of small gradients 
are negligible.}
The gradients of the samples in the batch are mathematically averaged to produce the update for each iteration of batch SGD without clipping. The gradient sizes for various samples within a batch may differ over orders of magnitude. Therefore, small gradient samples have little impact on the batch-averaged gradient for the entire batch. We calculate the cosine similarity between each sample's gradient and the actual batch-averaged gradient to determine how much each sample contributed to the final update.
Giving very large weights to small gradient samples will result in a significant difference between the normalized batch-averaged gradient and the unnormalized gradient, as shown in Figure~\ref{figure:fig1} where larger individual gradients maintain higher cosine similarity to the true batch average while smaller gradient samples are almost orthogonal or even negative to it. Additional datasets have produced similar results (Appendix D).

\subsubsection{Monotonic weights bring larger convergence errors.}
Recall that the update in Auto-S/NSGD is equivalent to 
using a weighted average of per-sample gradients:
\begin{equation}
G_{\rm batch} = \frac{1}{|B_t|} \sum_{i \in B_t} \tilde{g}_{t,i} = \frac{1}{|B_t|} \sum_{i \in B_t} w_{t,i} g_{t,i},\nonumber
\end{equation}
where $w_{t,i}$ is monotonically decreasing with respect to $\|g_{t,i}\|$, i.e. $w_{t,i} = 1/(\| g_{t,i} \| + r)$.
This leads to a larger learning rate for a smaller individual gradient.
As a result, in the later stages of the optimization process, the magnitude of the majority of individual gradients tends to zero, but the size of the update is still in the same order as that in the beginning, making steady convergence more challenging.
This intuition is also reflected in its theoretical analysis. 
The norm of the gradient in Auto-S/NSGD has an $O(r^{-1})$ non-vanishing upper bound, which cannot be reduced as the number of iterations increases.

\begin{figure}[t]
\centering
\includegraphics[width=0.9\columnwidth]{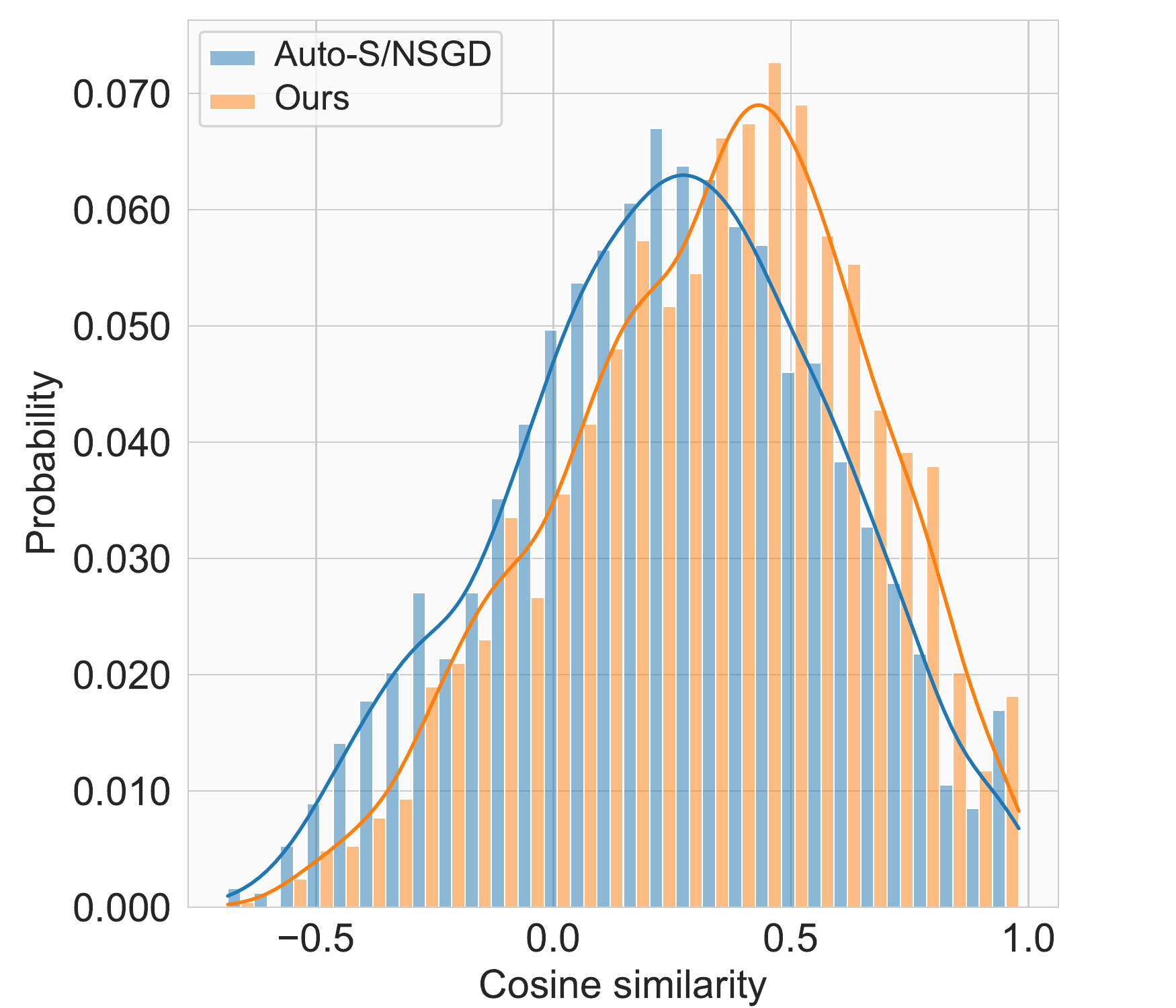} 
\caption{Cosine similarity histogram between the weighted batch-averaged gradients used in different methods and the real batch-averaged gradients.}
\label{figure:fig3}
\end{figure}

\subsection{Non-Monotonous Adaptive Weight Function}
We provide a non-monotonic adaptive weight function that gives small gradient samples a weight near 1 while weighting large gradients similarly to
$1/\|g_{t,i}\|$:
$$w(g_{t,i}) = 1/\left(\|g_{t,i}\|+ \frac{r}{\|g_{t,i}\|+ r}\right).$$
Our weight function can provide weights that are closer to automatic clipping when the gradient is large, as in Figure 2. Additionally, we restrict the gradient's weight to a certain order of magnitude when it is small in order to lessen the overall deviation.
We offer both theoretical and experimental evidence of the benefits of our adaptive weight function.

We describe our algorithmic pipeline and theoretical contributions in more detail in the following section.

\begin{algorithm}[tb]
\caption{DP-PSAC}
\label{alg:algorithm}
\textbf{Input}: initial weights $ x_0$ ,learning rate $ \eta_t $ , batch size $ B $, dataset $ \mathcal{S}= (z_1,...,z_N)  $, privacy budget $( \epsilon, \delta )$, max clipping threshold $C$, the number of iterations $T$ \\
\begin{algorithmic}[1] 
\STATE Compute the standard deviation $\sigma$ of noise based on Theorem \ref{theom:dp}
\FOR{ iteration $ t = 0,...,T-1$ }
\STATE Sample a batch $ \mathcal{D}_t := \{z^{t}_i\}^{b}_{i=1}$ from $\mathcal{S}$ uniformly with replacement
\STATE Compute the gradient $ g_{t,i}$ for each sample
\STATE $\displaystyle \widetilde{g}_{t,i} = Cg_{t,i}/\left(\|g_{t,i}\|+ \frac{r}{\|g_{t,i}\|+ r}\right)$
\STATE $\hat{g}_{t} = \sum\limits_{i = 1}^{b}\widetilde{g}_{t,i} + \mathcal{N}(0, C^2\sigma^2)$
\STATE $\displaystyle x_{t+1} = x_t - \frac{\eta_t}{B} \hat{g}_{t}$
\ENDFOR
\end{algorithmic}
\end{algorithm}

\section{Per-Sample Adaptive Clipping Training}
Here, we formally define the differentially private training algorithm DP-PSAC based on the per-sample adaptive clipping method.
In the $k$-th iteration, The $i$-th gradient $g_{t,i}$  is clipped as 
$$\displaystyle \widetilde{g}_{t,i} = {\rm Clip}(g_{t,i}) =  C g_{t,i}/\left(\|g_{t,i}\|+ \frac{r}{\|g_{t,i}\|+ r}\right),$$
where $C$ is the hyperparameter for clipping.
Then, we can define the clipping weight (scaling factor) as
$$w(g_{t,i}) = \frac{\widetilde{g}_{t,i}}{Cg_{t,i}} = 1/\left(\|g_{t,i}\|+ \frac{r}{\|g_{t,i}\|+ r}\right).$$

As the result, the model increment in the $t$-th iteration can be formulated as below:
{\small
\begin{eqnarray}\nonumber
&&\Delta x_t \\ \nonumber
&=& - \frac{\eta_t}{B} \left(\sum\limits_{i = 1}^{B}\widetilde{g}_{t,i} + \mathcal{N}(0, C^2\sigma^2)\right)\\ \nonumber
&=& - \frac{\eta_t}{B} \left(\sum\limits_{i = 1}^{B}Cg_{t,i}/\left(\|g_{t,i}\|+ \frac{r}{\|g_{t,i}\|+ r}\right) + \mathcal{N}(0, C^2\sigma^2)\right)\\ \nonumber
&=& - \frac{\eta_t C}{B} \left(\sum\limits_{i = 1}^{B}g_{t,i}/\left(\|g_{t,i}\|+ \frac{r}{\|g_{t,i}\|+ r}\right) + \mathcal{N}(0, \sigma^2)\right). \nonumber
\end{eqnarray}
}The clipping parameter $C$ does not require adjustment because it is coupled with the learning rate $\eta_t$, as can be seen from this equality. The entire procedure of per-sample adaptive gradient clipping-based differential privacy training is summarized in Algorithm~\ref{alg:algorithm}.

We compute the cosine similarity between the batch-averaged gradient that was weighted using different functions and the true batch-averaged gradient in the same iteration in order to determine the deviation between the two gradients. The greater the cosine similarity, the closer the two gradients are.
We run both our weight function and that in Auto-S/NSGD five times each on the FashionMNIST dataset, measuring the cosine similarity between the weighted batch-averaged gradient and the true batch-averaged gradient every 10 iterations.
As shown in Figure~\ref{figure:fig3}, compared with Auto-S/NSGD, our method has a higher percentage of gradients with larger similarity, which demonstrates that our method is statistically closer to the true batch-averaged gradient than Auto-S/NSGD.
Besides, for the ``lazy region'' problem of Auto-V~\cite{bu2022automatic}, we show that our method can solve this problem better than Auto-S through simulation experiments under the same setting (Appendix C).

\begin{table*}[t]
\centering

\begin{tabular}{l l l l}
\toprule
  \makecell[c]{Method}
  &
  \makecell[c]{Clipping threshold}
  &
  \makecell[c]{Additional assumption}
  &
  \makecell[c]{Non-vanishing bound}
  \\
  \midrule
 \makecell[c]{DP-SGD~\cite{yang2022normalized}}
 &
 \makecell[c]{Yes} 
 &
 \makecell[c]{$c>\frac{2\tau_0}{1-\tau_1}$} 
 &
 \makecell[c]{/} \\
 \makecell[cc]{Auto-S/NSGD}~\cite{bu2022automatic,yang2022normalized} &
 \makecell[cc]{No} &
 \makecell[cc]{$p(\Delta)=p(-\Delta)$ or $r>\tau_0$}&
 \makecell[cc]{$\mathcal{O}(r^{-1})$} \\
 \makecell[c]{DP-PSAC~(Ours)} &
 \makecell[c]{No} &
 \makecell[c]{/} &
 \makecell[c]{$\mathcal{O}(r^{-1/2})$} \\
 \bottomrule
\end{tabular}
\caption{Comparison of theoretical results of DP-SGD, Auto-S/NSGD and DP-PSAC.}
\label{table:theorem_comp}
\end{table*}

It should be highlighted that our method applies an adaptive norm constraint depending on the properties of each gradient sample, which is a novel and unexplored viewpoint.
Although Auto-S and NSGD are incredibly close to this perspective, they focus on scaling all gradient norms to the same or similar size, which limits their adaptability.

\subsection{Privacy Guarantee of DP-PSAC}

To achieve privacy protection, existing learning methods with DP such as DP-SGD mainly adopt two techniques, namely clipping gradients and adding Gaussian noise.
The first technique is used to limit the privacy sensitivity of gradients such that $\|g\| \leq C$ and the second technique is used to apply the Gaussian mechanism~\cite{dong2019gaussian} to achieve DP. 
We observe the per sample adaptive clipped gradient in DP-PSAC satisfies $\widetilde{g}_{t,i} =  C g_{t,i}/\left(\|g_{t,i}\|+ \frac{r}{\|g_{t,i}\|+ r}\right) \leq C$, which means that DP-PSAC can achieve the same privacy-sensitivity constraint for gradients as DP-SGD.
Furthermore, this means that the privacy analysis 
on DP-SGD still can be applied on DP-PSAC.

\begin{theorem}\label{theom:dp}
There exist constants $c_1$ and $c_2$ so that given the sampling probability $q=B/N$ and the number of iterations $T$, for any $\epsilon \leq c_1q^2T$ and $\delta > 0$, Algorithm 1 is $(\epsilon, \delta)$-differentially private if we choose
$$\sigma \geq c_2 \frac{q\sqrt{T {\rm log}(1/\delta)}}{\epsilon}.$$
\end{theorem}

\subsection{Convergence Guarantee of DP-PSAC}

Without losing generality, we prove that DP-PSAC converges to the stationary point, i.e. $\lim_{t \to +\infty}\| \nabla f(x_t) \| = 0$, which is widely adopted criterion for general non-convex optimization~\cite{ghadimi2013stochastic}. 
All detailed proofs are deferred to the appendix due to the page limitation.
We first give Theorem~\ref{theo:converge1} to bound the expected gradient norm with the number of iteration $T$ and the variance of the Gaussian noise $\sigma^2$.

\begin{theorem}\label{theo:converge1}
For $f(x)$ satisfying Assumptions 1, 2. Given an arbitrary noise multiplier $\sigma$ and constant $r \in (0,1]$, we run DP-PSAC for the number of iterations $T \geq A(L,\tau, d, r, \sigma, B)$ (Lemma~6 in Appendix. B) with a constant learning rate
$$\eta = \sqrt{\frac{2B^2}{d\sigma^2T(L_0+L_1(\tau_0+1))}}.$$
We can observe that the gradient norm can be bounded by the following inequality:

\begin{eqnarray}\nonumber 
\mathbb{E}(\mathop{min}\limits_{0 \leq t < T}\|\nabla f(x_t)\|) \leq \mathcal{O}\left(\sqrt[4]{\frac{d\sigma^2}{TB^2}} + \sqrt[4]{\frac{B^2}{Td\sigma^2}}\right)~~~~~~~~~~~~~~~~\\\nonumber
+\frac{8\tau_0^2(1+\tau_0)}{3N(\tau_0, \tau_1, r)(1-\tau_1)^3(\tau_0 + \frac{r(1-\tau_1)}{2\tau_0 + r(1-\tau_1)})(2\sqrt{r}-r)},
\end{eqnarray}
where 
{
\begin{small}
\begin{eqnarray}\nonumber
N(\tau_0, \tau_1, r) = \min \left(\frac{\tau_0}{1-\tau_1}, \frac{2\tau_0^2 + r\tau_0(1-\tau_1)}{4\tau_0^2 + 2r\tau_0(1-\tau_1)+r(1-\tau_1)^2}\right).\nonumber
\end{eqnarray}
\end{small}
}
\end{theorem}

It can be inferred from Theorem~\ref{theom:dp} that the noise multiplier $\sigma$ depends on the privacy parameters $(\epsilon, \delta)$ and the number of iterations $T$. In order to achieve the DP guarantee, Theorem~\ref{theo:converge1} can be extended to observe the following Corollary by properly setting $\sigma$.

\begin{corollary}\label{corollary:converge1}
With the same setting as Theorem \ref{theo:converge1},  we set $T \geq \mathcal{O}(N^2\epsilon^2/(d\log(1/\delta)))$.
To achieve $(\epsilon, \delta)$ DP guarantees with a sufficient number of samples $N \geq L_1A'(\epsilon, \delta, \tau, L, d, r)$ (Lemma~8 in Appendix. B), the expected gradient norm can be bounded as:
\begin{small}
\begin{eqnarray}\nonumber
\mathbb{E}(\mathop{min}\limits_{0 \leq t < T}||\nabla f(x_t)||) \leq \mathcal{O} \left( \sqrt{\frac{ \sqrt{d\log(1/\delta)}}{N\epsilon}}\right)~~~~~~~~~~~~~~~~~~~\\\nonumber
+\frac{8\tau_0^2(1+\tau_0)}{3N(\tau_0, \tau_1, r)(1-\tau_1)^3(\tau_0 + \frac{r(1-\tau_1)}{2\tau_0 + r(1-\tau_1)})(2\sqrt{r}-r)}.\nonumber
\end{eqnarray}
\end{small}
\end{corollary}

From Theorem~\ref{theo:converge1} and Corollary~\ref{corollary:converge1}, it can be observed that when we choose a suitable learning rate, DP-PSAC can achieve the convergence rate of $\mathcal{O}( \sqrt{\frac{ \sqrt{d\log(1/\delta)}}{N\epsilon}})$, which is consistent with the latest results of the differentially private non-convex optimization~\cite{bu2022automatic,yang2022normalized}.

\begin{table*}[t]
\centering
\begin{tabular}{l|l|l|l|l|l}
\toprule
    \makecell[c]{Task} & \makecell[c]{Model} & \makecell[c]{$(\epsilon,\delta)$} & \makecell[c]{DP-SGD$(\%)$} & \makecell[c]{Auto-S/NSGD$(\%)$} & \makecell[c]{DP-PSAC$(\%)$}\\
    \midrule
    \makecell[c]{MNIST} & \makecell[c]{CNN} & \makecell[c]{$(3,1e-5)$} & \makecell[c]{$98.12\pm 0.07$} & \makecell[c]{$98.17\pm 0.07$} &  \makecell[c]{$\bm{98.24\pm 0.07}$}\\
    \makecell[c]{FashionMNIST} & \makecell[c]{CNN} & \makecell[c]{$(3,1e-5)$} & \makecell[c]{$86.22\pm 0.29$} & \makecell[c]{$86.30\pm 0.21$} & \makecell[c]{$\bm{86.56\pm 0.16}$}\\
    \makecell[c]{CIFAR10} & \makecell[c]{SimCLRv2} &  \makecell[c]{$(2,1e-5)$} & {$92.47\pm 0.07$}& \makecell[c]{$92.72\pm 0.16$} & \makecell[c]{$\bm{92.78\pm 0.13}$}\\
    \makecell[c]{imagenette} & \makecell[c]{ResNet9} & \makecell[c]{$(8,1e-4)$} & \makecell[c]{$63.66\pm 0.05$} &  \makecell[c]{$63.44\pm 0.24$} & \makecell[c]{$\bm{64.00\pm 0.16}$} \\
    \makecell[c]{CelebA [Smiling]} & \makecell[c]{ResNet9} & \makecell[c]{$(8,5e-6)$} & \makecell[c]{$91.17\pm 0.06$} & \makecell[c]{$91.10 \pm 0.02$} & \makecell[c]{$\bm{91.41 \pm 0.02}$}\\
    \makecell[c]{CelebA [Male]} & \makecell[c]{ResNet9} & \makecell[c]{$(8,5e-6)$} & \makecell[c]{$95.46\pm 0.03$} & \makecell[c]{$95.48\pm 0.04$} & \makecell[c]{$\bm{95.57\pm 0.02}$}\\
    \makecell[c]{CelebA Multi-label} & \makecell[c]{ResNet9} & \makecell[c]{$(8,5e-6)$} & \makecell[c]{$88.56\pm0.04$} & \makecell[c]{$88.49\pm 0.10$} & \makecell[c]{$\bm{88.69\pm 0.01}$} \\
    \bottomrule
\end{tabular}
\caption{Test accuracy of DP-SGD, Auto-S and DP-PSAC on image classification tasks.}
\label{table2}
\vspace{-4mm}
\end{table*}

\begin{remark}
There are no additional assumptions to limit the hyperparameters or distribution of gradients in the convergence proof of Theorem~\ref{theo:converge1}. 
\end{remark}

In previous work, the convergence results of \cite{bu2022automatic} rely on the assumption that the gradient distribution is symmetric. The convergence results of \citet{yang2022normalized} depend on the assumption that the regularization term satisfies  $r>\tau_0$, but $\tau_0$ is difficult to observe. 
DP-PSAC does not rely on extra-assumed properties because its weight function is non-monotonic and there is a strict upper bound that does not depend on $\|\nabla f(x)\|$. 
We summarize the theoretical comparison of different algorithms in Table~\ref{table:theorem_comp}.
We demonstrate the theoretical superiority of this weight function by briefly introducing our proof procedure.

Similar to conventional non-convex optimization based on $(L_0, L_1){\rm -generalized~smooth}$, our convergence analysis is developed by the following lemma:
\begin{lemma}\label{lemma:iter}
Under the premise of Assumption 1, for each iteration $t$, letting $w_{t,i}=w(g_{t,i})$ indicate the sample weight function, the following inequality holds:
\begin{small}
\begin{eqnarray}\nonumber
\mathbb{E}_t[f(x_{t+1})]-f(x_t) \leq -\eta \mathbb{E}_t[\frac{1}{B}\sum\limits_{i=1}^B{\langle w_{t,i} \nabla f(x_t), g_{t,i}\rangle}]~~~~~ \\\nonumber
+ \mathbb{E}_t\frac{L_0+L_1\|\nabla f(x_t)\|}{2}\eta^2(\frac{d\sigma^2}{B^2} + \frac{1}{B}\sum\limits_{i=1}^B\|w_{t,i} g_{t,i} \|^2).\nonumber
\end{eqnarray}
\end{small}
\end{lemma}

For the first term, it can be scaled to $\mathcal{O}(\eta \|\nabla f(x_t)\|)$ (when $\|\nabla f(x_t)\| \geq \tau_0/(1-\tau_1)$) or $\mathcal{O}(\eta \| \nabla f(x_t)\|^2) + \mathcal{O}(\eta) )$ (when $\|\nabla f(x_t)\| < \tau_0/(1-\tau_1)$) by lemma~5 in Appendix A. For the second term, a suitable $\eta$ is chosen so that it can be  upper bounded by 
$\mathcal{O}(\eta^2)+\mathcal{O}(\eta w_t\|\nabla f(x_t)\|^2)$. 
Since $\mathcal{O}(\eta w_t\|\nabla f(x_t)\|^2)$ is consistent with the form of the first item, 
it can be similarly scaled
as the first term. 
At this point, we only need to take $\eta \propto 1/\sqrt{T}$, and
sum up
the above formula 
from $t=1$ to $T$
to deduce convergence result.

The hardest part of dealing with the second term is bounding $(L_0+L_1\|\nabla f(x_t)\|^2)w_t$ with a constant that does not depend on $\nabla f(x_t)$. Due to its monotonically decreasing weight function, NSGD can only find an upper bound that does not depend on $\nabla f(x_t)$ by assuming $r > \tau_0$. In our method, we can find a constant upper bound without making any additional assumptions by the following lemma.

\begin{lemma}\label{lemma2}
Under Assumption 2, for any $r \in (0,1]$, $w_{t,i} = 1/(\|g_{t,i}\| + r/(\|g_{t,i}\|+r))$, we have the following inequality:
\begin{small}
$$(L_0+L_1\|\nabla f(x_t)\|)w_{t,i} \leq max(\frac{L_0(1-\tau_1)+L_1(\sqrt{r}-r+\tau_0)}{(1-\tau_1)(2\sqrt{r}-r)}, $$ 
$$\frac{L_0(1-\tau_1) + L_1\tau_0 + L_1\sqrt{r}}{\sqrt{r}(1-\tau_1)}).$$
\end{small}
\end{lemma}

Since Lemma~\ref{lemma2} does not use any additional assumptions on $r$, any choice of $r \in (0,1]$ is feasible to achieve the the theoretical results in Corollary~\ref{corollary:converge1}.

\begin{remark}
Theorems \ref{theo:converge1} and Corollary \ref{corollary:converge1} give the non-vanishing bound in the order of $\mathcal{O}(r^{-1/2})$, which is superior compared with $\mathcal{O}(r^{-1})$ in NSGD~\cite{yang2022normalized}.
\end{remark}

The normalization-based method innovatively solves the problem that the clipping threshold is difficult to tune. But it introduces an immortal deviation to the optimization process, which cannot be eliminated by increasing the number of iterations or the privacy budget. 
At the same time, the upper bound of this deviation is inversely proportional to the multiplication of $r$, which is a constant from 0 to 1~(e.g. 0.01). Our method reduces the upper bound on immortality deviation from $\mathcal{O}(r^{-1})$ to $\mathcal{O}(r^{-1/2})$ by controlling the maximum weight of the weight function.

\section{Experiments}
We evaluate the effectiveness of the proposed algorithm on multiple datasets for both image and sentence classification.

\subsubsection{Hardware and software information}
All experiments are performed on a server with an Intel Xeon Platinum 8369B CPU, an NVIDIA A100 GPU, and 125GB memory. The operating system is Ubuntu 20.04 and the CUDA Toolkit version is 11.3. All computer vision experimental training procedures are implemented based on the latest versions of Pytorch and Opacus~\cite{opacus}. The natural language processing experiments are based on private-transformers~\cite{li2021large} of version 0.1.0, transformers of version 4.11.3, and the latest version of Pytorch.

\begin{figure*}
    \centering
    \includegraphics[width=0.63\columnwidth]{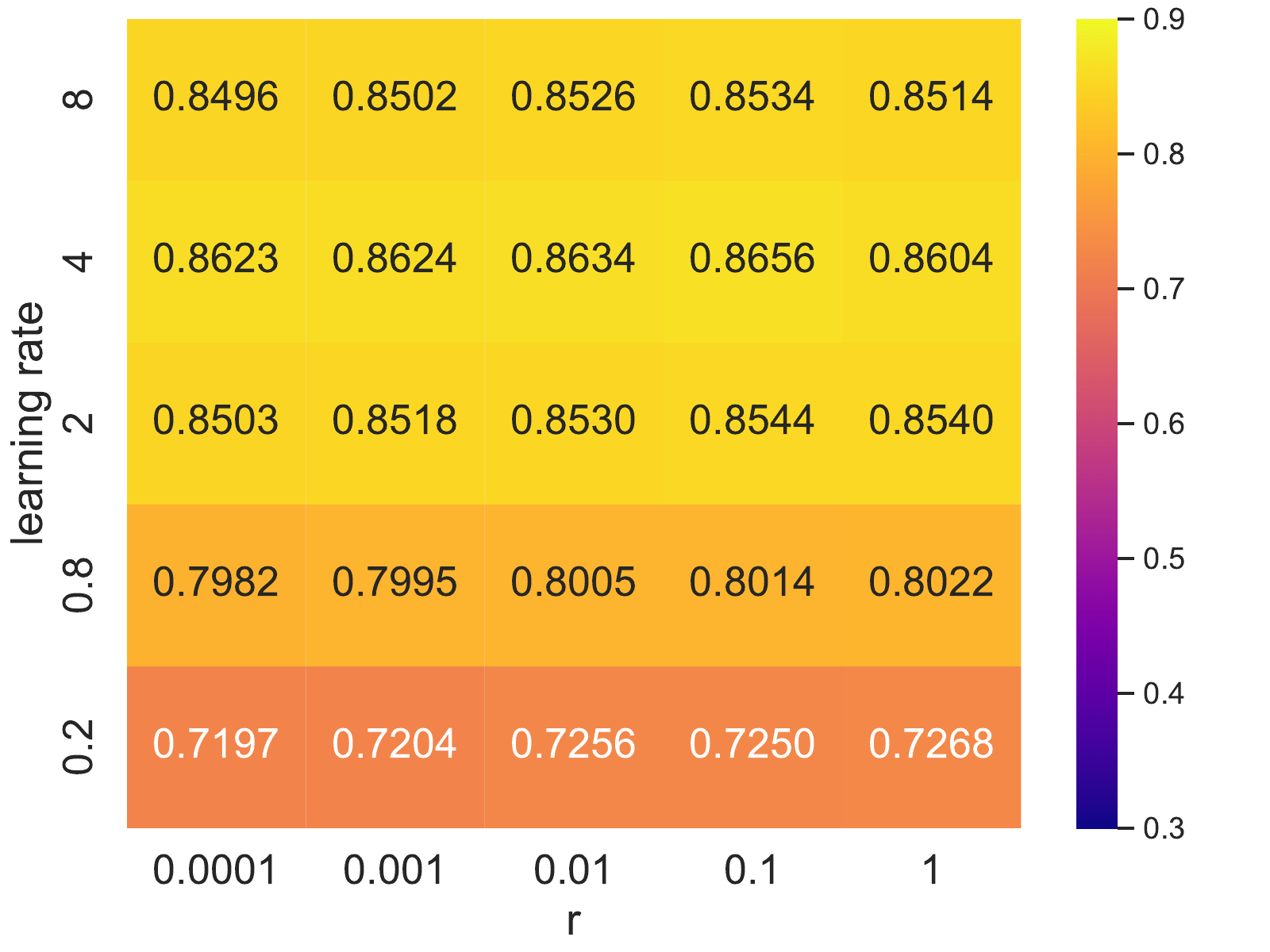}
    \includegraphics[width=0.63\columnwidth]{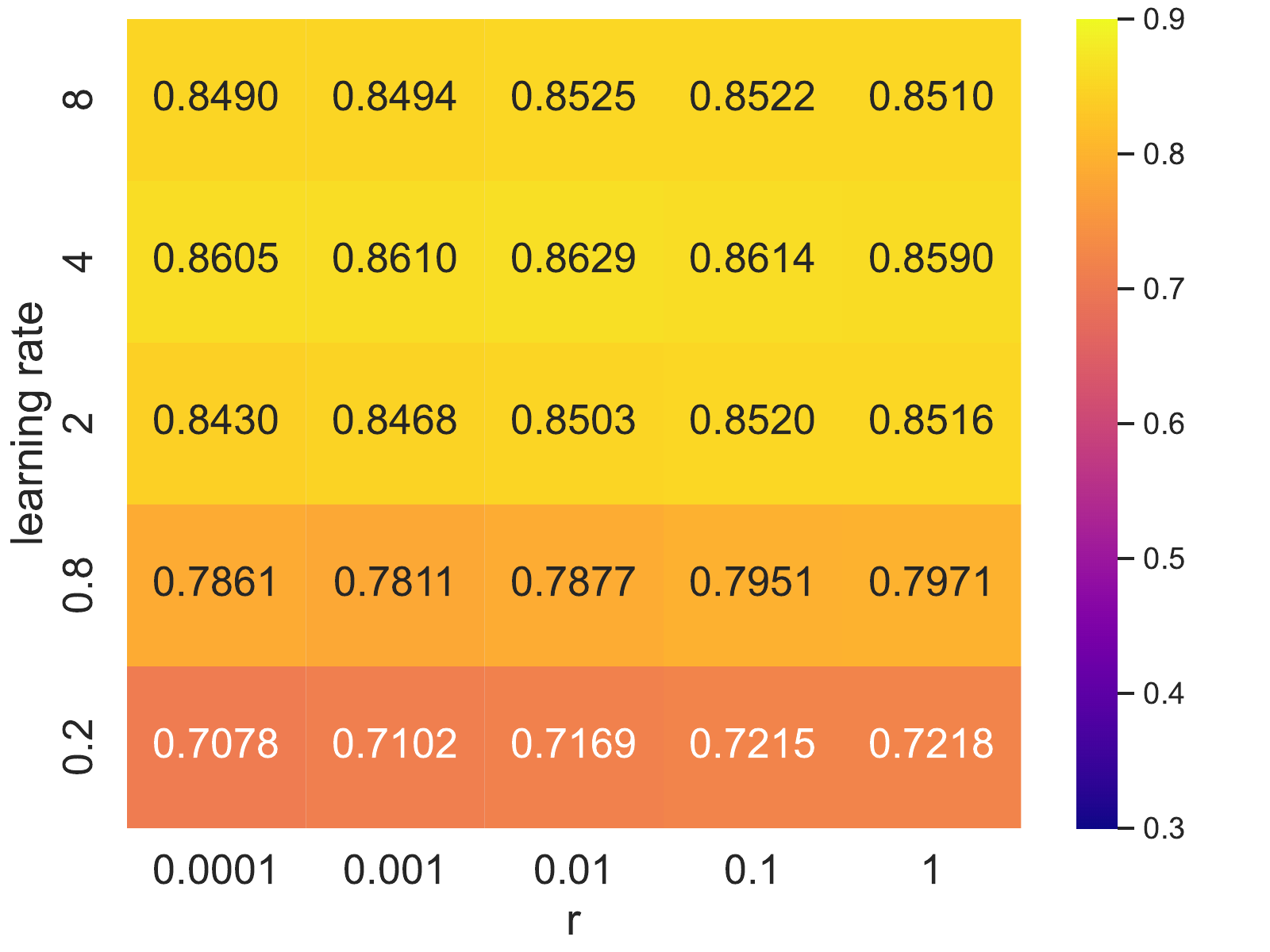}
    \includegraphics[width=0.63\columnwidth]{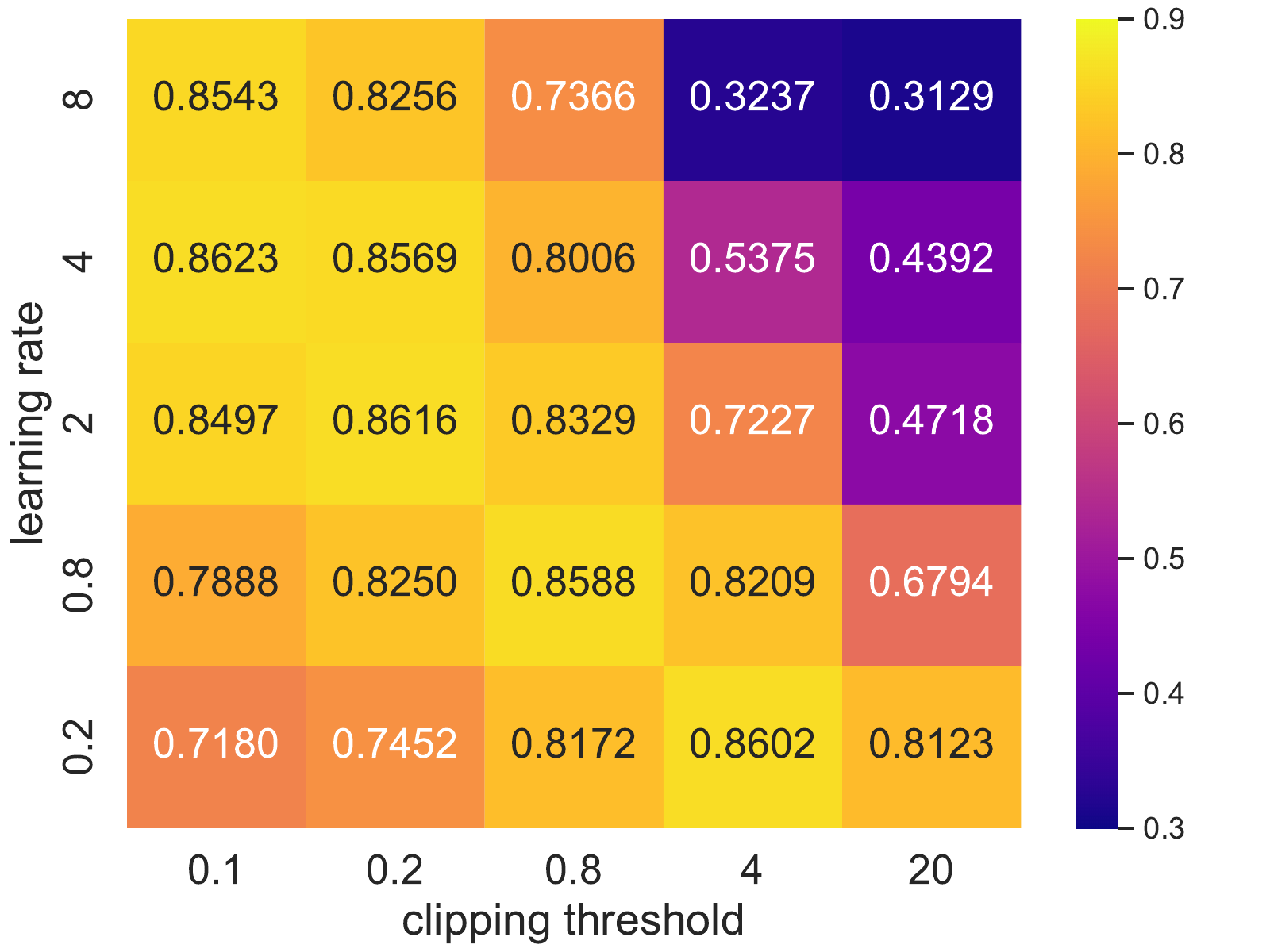}
    \caption{Test accuracy heatmap on the FashionMNIST task. Left: DP-PSAC. Middle: Auto-S/NSGD. Right: DP-SGD.}
    \label{fig:heatmap}
\end{figure*}

\begin{table*}[t]
\centering
\begin{tabular}{l|l|l|l|l|l|l|l|l}
\toprule
\makecell[c]{\multirow{2}{*}{Method}} & \multicolumn{4}{c|}{$\epsilon = 3$}& \multicolumn{4}{c}{$\epsilon = 8$} \\ \cmidrule{2-5} \cmidrule{6-9}
 &\makecell[c]{MNLI(m/mm)} &\makecell[c]{QQP} &\makecell[c]{QNLI} & \makecell[c]{SST-2} &\makecell[c]{MNLI(m/mm)} &\makecell[c]{QQP} &\makecell[c]{QNLI} & \makecell[c]{SST-2}\\
  \midrule
 \makecell[c]{DP-SGD~\cite{li2021large}} &\makecell[c]{82.45/82.99} & \makecell[c]{85.56}& \makecell[c]{87.42}& \makecell[c]{91.86} &\makecell[c]{83.20/83.46} & \makecell[c]{86.08}& \makecell[c]{87.94}& \makecell[c]{92.09}\\
\makecell[c]{Auto-S~\cite{bu2022automatic}} &\makecell[c]{\textbf{83.22}/83.21} & \makecell[c]{85.76}& \makecell[c]{86.91}& \makecell[c]{92.32} &\makecell[c]{\textbf{83.82}/83.55} & \makecell[c]{86.58}& \makecell[c]{87.85}& \makecell[c]{92.43}\\
\makecell[c]{DP-PSAC(Ours)} &\makecell[c]{82.74/\textbf{83.36}} & \makecell[c]{\textbf{85.83}}& \makecell[c]{\textbf{87.48}}& \makecell[c]{\textbf{92.43}} &\makecell[c]{83.65/\textbf{83.87}} & \makecell[c]{\textbf{86.60}}& \makecell[c]{\textbf{88.03}}& \makecell[c]{\textbf{92.55}} \\
 \bottomrule
\end{tabular}
\caption{
Test accuracy
of sentence classification for DP-SGD, Auto-S, and DP-PSAC with $\epsilon=3, 8$.}
\label{table3}
\vspace{-4mm}
\end{table*}

\subsection{Image Classification Task}

\subsubsection{Dataset}
We conduct extensive experiments on multiple image classification datasets, including MNIST~\cite{lecun1998gradient}, FashionMNIST ~\cite{xiao2017fashion}, CIFAR10~\cite{krizhevsky2009learning}, imagenette (a subset of imagenet~\cite{deng2009imagenet} with ten labels), and CelebA~\cite{liu2015faceattributes}.

\subsubsection{Method}
Our main comparison methods are DP-SGD and Auto-S/NSGD. 
For DP-SGD, we refer to the implementations of \citet{papernot2021tempered}, \citet{tramer2020differentially}, and \citet{klause2022differentially}, which achieves the state-of-the-art performance of Abadi's clipping-based DP learning on different image datasets. 
For Auto-S/NSGD, we adopt the same settings as ~\citet{bu2022automatic}, which exhibits the state-of-the-art differentially private optimization performance. 
Specifically, we train a four-layer CNN model on MNIST and FashionMNIST, which have the same settings as~\citet{tramer2020differentially}. 
Then for CIFAR10, we keep the same experimental setup as \citet{tramer2020differentially} and use pretrained SimCLRv2~\cite{chen2020simple} based on contrastive learning~\cite{chen2020simple, khosla2020supervised, cheng2023ssvmr}. 
Further, we train a ResNet9~\cite{he2016deep} model on imagenette and CelebA to validate the performance of our method on more complex multi-classification and multi-label classification problems, and the experimental setup for this part is the same as previous works~\cite{klause2022differentially,bu2022automatic}. We run all methods five times to get all of the results shown in Table 2.

\subsubsection{Result} Firstly, as shown in Figure~\ref{fig:heatmap}, we notice that the test accuracy changes very little with $r$ in DP-PSAC and Auto-S/NSGD for the same learning rate. 
Correspondingly, when using DP-SGD, the test accuracy is very sensitive to the clipping threshold $C$. This shows that the hyperparameter $r$ is more stable and easier to tune than the clipping threshold $C$. Usually, we only need to set r to a positive number not larger than 1, for instance, 0.1, to get a near-optimal result.
It can be observed from Table 2 that,  our method outperforms both DP-SGD and Auto-S in differentially private learning on the mainstream image classification datasets.
In particular, DP-PSAC is more robust than Auto-S/NSGD since it exhibits a lower level of variance. 
This corroborates with our theoretical result that DP-PSAC has a lower non-vanishing bound than Auto-S/NSGD. 
These evaluations show that our algorithm performs well on logistic regression, basic CNN, and ResNet, and its high performance is independent of any particular network architecture.

\subsection{Sentence Classification Task}

\subsubsection{Dataset}
We used four sentence classification datasets from the GLUE benchmark dataset, including MNLI (multi-genre inference)~\cite{williams2017broad}, QQP (equivalence classification), QNLI (Question-answering inference)~\cite{rajpurkar2016squad}, and SST-2 (sentiment classification)~\cite{socher2013recursive}.

\subsubsection{Method}
The code of the sentence classification experiment refers to \citet{li2021large}. 
In order to ensure the adequacy of the experiment, we use the roberta-base model to compare the full-parameter training performance of DP-PSAC, Auto-S/NSGD~\cite{bu2022automatic,yang2022normalized} and DP-SGD~\cite{li2021large} on four different datasets under large($\epsilon = 3$) and small($\epsilon = 8$) noise conditions, respectively. The test accuracy for DP-SGD and Auto-S are taken from~\cite{li2021large} and~\cite{bu2022automatic}, respectively.


\subsubsection{Result} 
Table 3 shows that DP-PSAC performs better than or similar to the best baseline in both small and large noise conditions. 
Specifically, on the MNLI dataset, our method outperforms Auto-S/NSGD on the MNLI-mm test set, 
which is not independent and identically distributed with the training set, 
and outperforming DP-SGD on both MNLI-m and MNLI-mm. 
For the QQP dataset, a sentence classification dataset with uneven sample distribution, DP-PSAC achieves higher accuracy than the baselines. 
Although Auto-S/NSGD does not achieve better results than DP-SGD on the QNLI dataset, our method, as an improvement of Auto-S/NSGD, achieves the latest state-of-the-art.
Meanwhile, on the SST-2 dataset, our method not only achieves better accuracy but also enables our model performance at $\epsilon=3$ to reach the previous state-of-the-art at $\epsilon=8$.

\section{Conclusion}
In this study, we propose a differentially private optimization approach with per-sample adaptive clipping, which can reduce deviation by giving gradients different weights according to their magnitudes while preserving privacy constraints. Without making any extrinsic assumptions, we investigate the convergence of DP-PSAC in non-convex scenarios and demonstrate that it offers a reduced upper bound on indestructible deviation than Auto-S/NSGD. Experimental results demonstrate that DP-PSAC accomplishes the state-of-the-art in differentially private optimization on both language and computer vision problems.

Per-sample adaptive clipping is a new perspective, which is different from adaptive clipping with iterations~\cite{du2021dynamic,andrew2021differentially} and per-axis adaptation~\cite{asi2021private}. 
In future work, we will consider to develop a data-driven adaptive weight function and more realistic application scenarios, such as resource offloading, flow detection and speech task\cite{yao2022block, zhou2022aneff, cheng2022m3st, zhu2022dynamic}.

\clearpage
\newpage

\section{Acknowledgements}
This work was in part supported by the 
National Key Research and Development Program of China under Grant 
2022YFB3102301, the China National Funds for Distinguished Young Scientists with No. 61825204, the NSFC Project with No. 61932016, No. 62101301, No. 62132011, and No. 62132009, the Beijing Outstanding Young Scientist Program with No. BJJWZYJH01201910003011, National Natural Science Foundation of China (U22B2031), 
CCF-AFSG Research Fund (CCF-AFSG RF20210023) , China Computer Federation (CCF)-Huawei Populus euphratica forest fund (CCF-HuaweiBC2021005), Chinese Association for Artifificial Intelligence (CAAI)-Huawei MindSpore Open Fund (CAAIXSJLJJ-2020-014A).

The author thanks Kai Xiao, Weiqiang Wang, as well as the reviewers/SPC/AC for the constructive comments to improve the paper.

\bibliography{aaai23}

\begin{thebibliography}{51}
\providecommand{\natexlab}[1]{#1}

\bibitem[{Abadi et~al.(2016)Abadi, Chu, Goodfellow, McMahan, Mironov, Talwar,
  and Zhang}]{abadi2016deep}
Abadi, M.; Chu, A.; Goodfellow, I.; McMahan, H.~B.; Mironov, I.; Talwar, K.;
  and Zhang, L. 2016.
\newblock Deep learning with differential privacy.
\newblock In \emph{Proceedings of the 2016 ACM SIGSAC conference on computer
  and communications security}, 308--318.

\bibitem[{Andrew et~al.(2021)Andrew, Thakkar, McMahan, and
  Ramaswamy}]{andrew2021differentially}
Andrew, G.; Thakkar, O.; McMahan, B.; and Ramaswamy, S. 2021.
\newblock Differentially private learning with adaptive clipping.
\newblock \emph{Advances in Neural Information Processing Systems}, 34:
  17455--17466.

\bibitem[{Asi et~al.(2021)Asi, Duchi, Fallah, Javidbakht, and
  Talwar}]{asi2021private}
Asi, H.; Duchi, J.; Fallah, A.; Javidbakht, O.; and Talwar, K. 2021.
\newblock Private adaptive gradient methods for convex optimization.
\newblock In \emph{International Conference on Machine Learning}, 383--392.
  PMLR.

\bibitem[{Bassily, Guzm{\'a}n, and Menart(2021)}]{bassily2021differentially}
Bassily, R.; Guzm{\'a}n, C.; and Menart, M. 2021.
\newblock Differentially private stochastic optimization: New results in convex
  and non-convex settings.
\newblock \emph{Advances in Neural Information Processing Systems}, 34:
  9317--9329.

\bibitem[{Bu et~al.(2022)Bu, Wang, Zha, and Karypis}]{bu2022automatic}
Bu, Z.; Wang, Y.-X.; Zha, S.; and Karypis, G. 2022.
\newblock Automatic Clipping: Differentially Private Deep Learning Made Easier
  and Stronger.
\newblock arXiv:2206.07136.

\bibitem[{Chen et~al.(2020)Chen, Kornblith, Norouzi, and
  Hinton}]{chen2020simple}
Chen, T.; Kornblith, S.; Norouzi, M.; and Hinton, G. 2020.
\newblock A simple framework for contrastive learning of visual
  representations.
\newblock In \emph{International conference on machine learning}, 1597--1607.
  PMLR.

\bibitem[{Cheng et~al.(2022)Cheng, Dong, Yue, Ko, Wang, and
  Zou}]{cheng2022m3st}
Cheng, X.; Dong, Q.; Yue, F.; Ko, T.; Wang, M.; and Zou, Y. 2022.
\newblock M3ST: Mix at Three Levels for Speech Translation.
\newblock \emph{arXiv preprint arXiv:2212.03657}.

\bibitem[{Cheng et~al.(2023)Cheng, Zhu, Li, Li, and Zou}]{cheng2023ssvmr}
Cheng, X.; Zhu, Z.; Li, H.; Li, Y.; and Zou, Y. 2023.
\newblock SSVMR: Saliency-based Self-training for Video-Music Retrieval.
\newblock \emph{arXiv preprint arXiv:2302.09328}.

\bibitem[{Deng et~al.(2009)Deng, Dong, Socher, Li, Li, and
  Fei-Fei}]{deng2009imagenet}
Deng, J.; Dong, W.; Socher, R.; Li, L.-J.; Li, K.; and Fei-Fei, L. 2009.
\newblock Imagenet: A large-scale hierarchical image database.
\newblock In \emph{2009 IEEE conference on computer vision and pattern
  recognition}, 248--255. Ieee.

\bibitem[{Dong, Roth, and Su(2019)}]{dong2019gaussian}
Dong, J.; Roth, A.; and Su, W.~J. 2019.
\newblock Gaussian differential privacy.
\newblock arXiv:1905.02383.

\bibitem[{Du et~al.(2021)Du, Li, Feng, and Chen}]{du2021dynamic}
Du, J.; Li, S.; Feng, M.; and Chen, S. 2021.
\newblock Dynamic differential-privacy preserving sgd.
\newblock arXiv:2111.00173.

\bibitem[{Dwork, Roth et~al.(2014)}]{dwork2014algorithmic}
Dwork, C.; Roth, A.; et~al. 2014.
\newblock The algorithmic foundations of differential privacy.
\newblock \emph{Foundations and Trends{\textregistered} in Theoretical Computer
  Science}, 9(3--4): 211--407.

\bibitem[{Esipova et~al.(2022)Esipova, Ghomi, Luo, and
  Cresswell}]{esipova2022disparate}
Esipova, M.~S.; Ghomi, A.~A.; Luo, Y.; and Cresswell, J.~C. 2022.
\newblock Disparate Impact in Differential Privacy from Gradient Misalignment.
\newblock \emph{arXiv preprint arXiv:2206.07737}.

\bibitem[{Ghadimi and Lan(2013)}]{ghadimi2013stochastic}
Ghadimi, S.; and Lan, G. 2013.
\newblock Stochastic first-and zeroth-order methods for nonconvex stochastic
  programming.
\newblock \emph{SIAM Journal on Optimization}, 23(4): 2341--2368.

\bibitem[{He et~al.(2016)He, Zhang, Ren, and Sun}]{he2016deep}
He, K.; Zhang, X.; Ren, S.; and Sun, J. 2016.
\newblock Deep residual learning for image recognition.
\newblock In \emph{Proceedings of the IEEE conference on computer vision and
  pattern recognition}, 770--778.

\bibitem[{Khosla et~al.(2020)Khosla, Teterwak, Wang, Sarna, Tian, Isola,
  Maschinot, Liu, and Krishnan}]{khosla2020supervised}
Khosla, P.; Teterwak, P.; Wang, C.; Sarna, A.; Tian, Y.; Isola, P.; Maschinot,
  A.; Liu, C.; and Krishnan, D. 2020.
\newblock Supervised contrastive learning.
\newblock \emph{Advances in neural information processing systems}, 33:
  18661--18673.

\bibitem[{Klause et~al.(2022)Klause, Ziller, Rueckert, Hammernik, and
  Kaissis}]{klause2022differentially}
Klause, H.; Ziller, A.; Rueckert, D.; Hammernik, K.; and Kaissis, G. 2022.
\newblock Differentially private training of residual networks with scale
  normalisation.
\newblock arXiv:2203.00324.

\bibitem[{Krizhevsky, Hinton et~al.(2009)}]{krizhevsky2009learning}
Krizhevsky, A.; Hinton, G.; et~al. 2009.
\newblock Learning multiple layers of features from tiny images.

\bibitem[{Kurakin et~al.(2022)Kurakin, Chien, Song, Geambasu, Terzis, and
  Thakurta}]{kurakin2022toward}
Kurakin, A.; Chien, S.; Song, S.; Geambasu, R.; Terzis, A.; and Thakurta, A.
  2022.
\newblock Toward training at imagenet scale with differential privacy.
\newblock arXiv:2201.12328.

\bibitem[{Kuru et~al.(2022)Kuru, Birbil, G{\"{u}}rb{\"{u}}zbalaban, and
  Yildirim}]{kuru2022differentially}
Kuru, N.; Birbil, S.~I.; G{\"{u}}rb{\"{u}}zbalaban, M.; and Yildirim, S. 2022.
\newblock Differentially private accelerated optimization algorithms.
\newblock \emph{SIAM Journal on Optimization}, 32(2): 795--821.

\bibitem[{LeCun et~al.(1998)LeCun, Bottou, Bengio, and
  Haffner}]{lecun1998gradient}
LeCun, Y.; Bottou, L.; Bengio, Y.; and Haffner, P. 1998.
\newblock Gradient-based learning applied to document recognition.
\newblock \emph{Proceedings of the IEEE}, 86(11): 2278--2324.

\bibitem[{Li et~al.(2022)Li, Zaheer, Reddi, and Smith}]{li2022private}
Li, T.; Zaheer, M.; Reddi, S.; and Smith, V. 2022.
\newblock Private adaptive optimization with side information.
\newblock In \emph{International Conference on Machine Learning}, 13086--13105.
  PMLR.

\bibitem[{Li et~al.(2021)Li, Tramer, Liang, and Hashimoto}]{li2021large}
Li, X.; Tramer, F.; Liang, P.; and Hashimoto, T. 2021.
\newblock Large language models can be strong differentially private learners.
\newblock arXiv:2110.05679.

\bibitem[{Liu et~al.(2015)Liu, Luo, Wang, and Tang}]{liu2015faceattributes}
Liu, Z.; Luo, P.; Wang, X.; and Tang, X. 2015.
\newblock Deep Learning Face Attributes in the Wild.
\newblock In \emph{Proceedings of International Conference on Computer Vision
  (ICCV)}.

\bibitem[{Mangold et~al.(2022)Mangold, Bellet, Salmon, and
  Tommasi}]{mangold2022differentially}
Mangold, P.; Bellet, A.; Salmon, J.; and Tommasi, M. 2022.
\newblock Differentially private coordinate descent for composite empirical
  risk minimization.
\newblock In \emph{International Conference on Machine Learning}, 14948--14978.
  PMLR.

\bibitem[{Mironov(2017)}]{mironov2017renyi}
Mironov, I. 2017.
\newblock R{\'e}nyi differential privacy.
\newblock In \emph{2017 IEEE 30th computer security foundations symposium
  (CSF)}, 263--275. IEEE.

\bibitem[{Papernot and Steinke(2021)}]{papernot2021hyperparameter}
Papernot, N.; and Steinke, T. 2021.
\newblock Hyperparameter Tuning with Renyi Differential Privacy.
\newblock In \emph{International Conference on Learning Representations}.

\bibitem[{Papernot et~al.(2021)Papernot, Thakurta, Song, Chien, and
  Erlingsson}]{papernot2021tempered}
Papernot, N.; Thakurta, A.; Song, S.; Chien, S.; and Erlingsson, {\'U}. 2021.
\newblock Tempered sigmoid activations for deep learning with differential
  privacy.
\newblock In \emph{Proceedings of the AAAI Conference on Artificial
  Intelligence}, volume~35, 9312--9321.

\bibitem[{Pichapati et~al.(2019)Pichapati, Suresh, Yu, Reddi, and
  Kumar}]{pichapati2019adaclip}
Pichapati, V.; Suresh, A.~T.; Yu, F.~X.; Reddi, S.~J.; and Kumar, S. 2019.
\newblock AdaCliP: Adaptive clipping for private SGD.
\newblock arXiv:1908.07643.

\bibitem[{Rajpurkar et~al.(2016)Rajpurkar, Zhang, Lopyrev, and
  Liang}]{rajpurkar2016squad}
Rajpurkar, P.; Zhang, J.; Lopyrev, K.; and Liang, P. 2016.
\newblock Squad: 100,000+ questions for machine comprehension of text.
\newblock arXiv:1606.05250.

\bibitem[{Socher et~al.(2013)Socher, Perelygin, Wu, Chuang, Manning, Ng, and
  Potts}]{socher2013recursive}
Socher, R.; Perelygin, A.; Wu, J.; Chuang, J.; Manning, C.~D.; Ng, A.~Y.; and
  Potts, C. 2013.
\newblock Recursive deep models for semantic compositionality over a sentiment
  treebank.
\newblock In \emph{Proceedings of the 2013 conference on empirical methods in
  natural language processing}, 1631--1642.

\bibitem[{Tramer and Boneh(2020)}]{tramer2020differentially}
Tramer, F.; and Boneh, D. 2020.
\newblock Differentially Private Learning Needs Better Features (or Much More
  Data).
\newblock In \emph{International Conference on Learning Representations}.

\bibitem[{van~der Veen et~al.(2018)van~der Veen, Seggers, Bloem, and
  Patrini}]{van2018three}
van~der Veen, K.~L.; Seggers, R.; Bloem, P.; and Patrini, G. 2018.
\newblock Three tools for practical differential privacy.
\newblock arXiv:1812.02890.

\bibitem[{Wang, Chen, and Xu(2019)}]{wang2019differentially}
Wang, D.; Chen, C.; and Xu, J. 2019.
\newblock Differentially private empirical risk minimization with non-convex
  loss functions.
\newblock In \emph{International Conference on Machine Learning}, 6526--6535.
  PMLR.

\bibitem[{Wang, Ye, and Xu(2017)}]{wang2017differentially}
Wang, D.; Ye, M.; and Xu, J. 2017.
\newblock Differentially private empirical risk minimization revisited: Faster
  and more general.
\newblock \emph{Advances in Neural Information Processing Systems}, 30.

\bibitem[{Wang et~al.(2022)Wang, Lei, Ying, and Zhang}]{wang2022differentially}
Wang, P.; Lei, Y.; Ying, Y.; and Zhang, H. 2022.
\newblock Differentially private SGD with non-smooth losses.
\newblock \emph{Applied and Computational Harmonic Analysis}, 56: 306--336.

\bibitem[{Williams, Nangia, and Bowman(2017)}]{williams2017broad}
Williams, A.; Nangia, N.; and Bowman, S.~R. 2017.
\newblock A broad-coverage challenge corpus for sentence understanding through
  inference.
\newblock arXiv:1704.05426.

\bibitem[{Wu et~al.(2021)Wu, Wang, Cristali, Gu, and Willett}]{wu2021adaptive}
Wu, X.; Wang, L.; Cristali, I.; Gu, Q.; and Willett, R. 2021.
\newblock Adaptive Differentially Private Empirical Risk Minimization.
\newblock arXiv:2110.07435.

\bibitem[{Xiao, Rasul, and Vollgraf(2017)}]{xiao2017fashion}
Xiao, H.; Rasul, K.; and Vollgraf, R. 2017.
\newblock Fashion-mnist: a novel image dataset for benchmarking machine
  learning algorithms.
\newblock arXiv:1708.07747.

\bibitem[{Yang et~al.(2022)Yang, Zhang, Chen, and Liu}]{yang2022normalized}
Yang, X.; Zhang, H.; Chen, W.; and Liu, T.-Y. 2022.
\newblock Normalized/Clipped SGD with Perturbation for Differentially Private
  Non-Convex Optimization.
\newblock arXiv:2206.13033.

\bibitem[{Yao et~al.(2022)Yao, Wang, Qu, Zhang, Zhang, Xu, and
  Xu}]{yao2022block}
Yao, S.; Wang, M.; Qu, Q.; Zhang, Z.; Zhang, Y.-F.; Xu, K.; and Xu, M. 2022.
\newblock Blockchain-Empowered Collaborative Task Offloading for
  Cloud-Edge-Device Computing.
\newblock \emph{IEEE Journal on Selected Areas in Communications}, 40:
  3485--3500.

\bibitem[{Yousefpour et~al.(2021)Yousefpour, Shilov, Sablayrolles, Testuggine,
  Prasad, Malek, Nguyen, Ghosh, Bharadwaj, Zhao, Cormode, and Mironov}]{opacus}
Yousefpour, A.; Shilov, I.; Sablayrolles, A.; Testuggine, D.; Prasad, K.;
  Malek, M.; Nguyen, J.; Ghosh, S.; Bharadwaj, A.; Zhao, J.; Cormode, G.; and
  Mironov, I. 2021.
\newblock Opacus: {U}ser-Friendly Differential Privacy Library in {PyTorch}.
\newblock arXiv:2109.12298.

\bibitem[{Yu et~al.(2022)Yu, Naik, Backurs, Gopi, Inan, Kamath, Kulkarni, Lee,
  Manoel, Wutschitz, Yekhanin, and Zhang}]{yu2022differentially}
Yu, D.; Naik, S.; Backurs, A.; Gopi, S.; Inan, H.~A.; Kamath, G.; Kulkarni, J.;
  Lee, Y.~T.; Manoel, A.; Wutschitz, L.; Yekhanin, S.; and Zhang, H. 2022.
\newblock Differentially Private Fine-tuning of Language Models.
\newblock In \emph{International Conference on Learning Representations}.

\bibitem[{Yu et~al.(2021{\natexlab{a}})Yu, Zhang, Chen, Yin, and
  Liu}]{yu2021gradient}
Yu, D.; Zhang, H.; Chen, W.; Yin, J.; and Liu, T.-Y. 2021{\natexlab{a}}.
\newblock Gradient perturbation is underrated for differentially private convex
  optimization.
\newblock In \emph{Proceedings of the Twenty-Ninth International Conference on
  International Joint Conferences on Artificial Intelligence}, 3117--3123.

\bibitem[{Yu et~al.(2021{\natexlab{b}})Yu, Zhang, Chen, Yin, and
  Liu}]{yu2021LargeScale}
Yu, D.; Zhang, H.; Chen, W.; Yin, J.; and Liu, T.-Y. 2021{\natexlab{b}}.
\newblock Large Scale Private Learning via Low-rank Reparametrization.
\newblock In Meila, M.; and Zhang, T., eds., \emph{Proceedings of the 38th
  International Conference on Machine Learning}, volume 139 of
  \emph{Proceedings of Machine Learning Research}, 12208--12218.

\bibitem[{Zhang et~al.(2020{\natexlab{a}})Zhang, Jin, Fang, and
  Wang}]{zhang2020advances}
Zhang, B.; Jin, J.; Fang, C.; and Wang, L. 2020{\natexlab{a}}.
\newblock Improved Analysis of Clipping Algorithms for Non-convex Optimization.
\newblock In Larochelle, H.; Ranzato, M.; Hadsell, R.; Balcan, M.; and Lin, H.,
  eds., \emph{Advances in Neural Information Processing Systems}, volume~33,
  15511--15521. Curran Associates, Inc.

\bibitem[{Zhang et~al.(2020{\natexlab{b}})Zhang, He, Sra, and
  Jadbabaie}]{Zhang2020Why}
Zhang, J.; He, T.; Sra, S.; and Jadbabaie, A. 2020{\natexlab{b}}.
\newblock Why Gradient Clipping Accelerates Training: A Theoretical
  Justification for Adaptivity.
\newblock In \emph{International Conference on Learning Representations}.

\bibitem[{Zhang, Ji, and Wang(2018)}]{zhang2018differentially}
Zhang, X.; Ji, S.; and Wang, T. 2018.
\newblock Differentially private releasing via deep generative model (technical
  report).
\newblock arXiv:1801.01594.

\bibitem[{Zhou et~al.(2023)Zhou, Liu, Fu, Li, and Xu}]{zhou2022aneff}
Zhou, G.; Liu, Z.; Fu, C.; Li, Q.; and Xu, K. 2023.
\newblock An Efficient Design of Intelligent Network Data Plane.
\newblock In \emph{32nd USENIX Security Symposium (USENIX Security 23)}.
  Anaheim, CA: USENIX Association.

\bibitem[{Zhu, Liu, and Han(2019)}]{zhu2019deep}
Zhu, L.; Liu, Z.; and Han, S. 2019.
\newblock Deep leakage from gradients.
\newblock \emph{Advances in neural information processing systems}, 32.

\bibitem[{Zhu et~al.(2022)Zhu, Xu, Cheng, Song, and Zou}]{zhu2022dynamic}
Zhu, Z.; Xu, W.; Cheng, X.; Song, T.; and Zou, Y. 2022.
\newblock A Dynamic Graph Interactive Framework with Label-Semantic Injection
  for Spoken Language Understanding.
\newblock \emph{arXiv preprint arXiv:2211.04023}.

\end{thebibliography}

\onecolumn
\appendix
\section{Prerequisite Lemmas}

\begin{lemma}\label{lemma:iter}
Under the premise of Assumption 1, for each iteration $t$, 
letting $w_{t,i}$ indicate the sample weight function, the following inequality holds:
\begin{equation} \nonumber
\mathbb{E}_t[f(x_{t+1})]-f(x_t) \leq \frac{1}{B}\sum\limits_{i=1}^B \left(-\eta_t \mathbb{E}_t \left \langle w_{t,i} \nabla f(x_t), g_{t,i} \right\rangle + \frac{L_0+L_1\|\nabla f(x_t)\|}{2}\eta_t^2 \left(\frac{d\sigma^2}{B^2} + \mathbb{E}_t\|w_{t,i} g_{t,i} \|^2 \right)\right).
\end{equation}

\end{lemma}

\begin{proof}

When $f(x)$ satisfies $(L_0,L_1)$-generalized smooth, 
for any $x,y \in \mathbb{R}^d$, 
we can obtain the following inequality from Lemma A.1~\cite{yang2022normalized}:
\begin{equation}\label{l_smooth}
f(y)\leq f(x)+ \langle \nabla f(x), y-x\rangle + \frac{L_0+L_1\|\nabla f(x)\|}{2}\|y-x\|^2,
\end{equation}
which is the basic result of $(L_0,L_1)$-generalized smooth.
Since the clipping threshold $C$ is coupled to the learning rate $\eta_t$, we set $C=1$. As a result, the update rule of DP-PSAC can be simplified to
\begin{equation}\label{update_rule}
x_{t+1} = x_t - \frac{\eta_t}{B}(\sum\limits_{i=1}^B w_{t,i}g_{t,i}+\mathcal{N}(0, \sigma^2)).    
\end{equation}
Combining (\ref{l_smooth}) and (\ref{update_rule}), we have 
{\small
\begin{equation}
  f(x_{t+1}) - f(x_t) \leq -\eta_t \left \langle \nabla f(x_t), \frac{1}{B}\left( \sum\limits_{i=1}^B w_{t,i}g_{t,i}+\mathcal{N}(0, \sigma^2) \right) \right\rangle + \frac{\eta_t^2\left(L_0+L_1\|\nabla f(x_t)\|\right)}{2} \left \| \frac{1}{B}\left(\sum\limits_{i=1}^B w_{t,i}g_{t,i}+\mathcal{N}(0, \sigma^2) \right) \right\|^2. 
\end{equation}
}Taking expectation with respect to the randomness in iteration $t$ gives
{\small
\begin{equation}\label{overall_expectaton}
  \mathbb{E}_t[f(x_{t+1})] - f(x_t) \leq - \eta_t \mathbb{E}_t  \left \langle \nabla f(x_t), \frac{1}{B}\left(\sum\limits_{i=1}^B w_{t,i}g_{t,i}\right)\right\rangle +  \frac{\eta_t^2\left(L_0+L_1\|\nabla f(x_t)\|\right)}{2} \mathbb{E}_t \left\| \frac{1}{B}(\sum\limits_{i=1}^B w_{t,i}g_{t,i}+\mathcal{N}(0, \sigma^2))\right\|^2.  
\end{equation}
}Next, we bound the last term as
\begin{eqnarray} \label{bound_variance}
\mathbb{E}_t \left\| \frac{1}{B} \left( \sum\limits_{i=1}^B w_{t,i}g_{t,i} + \mathcal{N}(0, \sigma^2) \right) \right \|^2  & = &  \frac{1}{B^2} \mathbb{E}_t \left( \left\|\sum\limits_{i=1}^B w_{t,i}g_{t,i} \right\|^2 + d\sigma^2 \right)\nonumber\\ 
&=& \frac{1}{B^2} \mathbb{E}_t \left\|\sum\limits_{i=1}^B w_{t,i}g_{t,i} \right\|^2 + \frac{d\sigma^2}{B^2} \nonumber\\ 
&\leq& \frac{1}{B}  \sum_{i=1}^B \mathbb{E}_t \| w_{t,i}g_{t,i}\|^2 + \frac{d\sigma^2}{B^2},
\end{eqnarray}
where the last inequality follows from the Cauchy inequality. Substituting (\ref{bound_variance}) into (\ref{overall_expectaton}), we deduce that
\begin{equation}
\mathbb{E}_t[f(x_{t+1})]-f(x_t) \leq \frac{1}{B}\sum\limits_{i=1}^B \left(-\eta_t \mathbb{E}_t \left \langle w_{t,i} \nabla f(x_t), g_{t,i} \right\rangle + \frac{L_0+L_1\|\nabla f(x_t)\|}{2}\eta_t^2 \left(\frac{d\sigma^2}{B^2} + \mathbb{E}_t\|w_{t,i} g_{t,i} \|^2 \right)\right).
\end{equation}
\end{proof}

~~

~~

~~

\begin{lemma}\label{lemma2}
Under Assumption 2, for any $r \in (0,1]$, $w_{t,i} = 1/(\|g_{t,i}\| + r/(\|g_{t,i}\|+r))$, we have the following inequality:

\begin{equation}\nonumber
    (L_0+L_1\|\nabla f(x_t)\|)w_{t,i} \leq \max \left(\frac{L_0(1-\tau_1)+L_1(\sqrt{r}-r+\tau_0)}{(1-\tau_1)(2\sqrt{r}-r)},\frac{L_0(1-\tau_1) + L_1\tau_0 + L_1\sqrt{r}}{\sqrt{r}(1-\tau_1)}\right).
\end{equation}

\end{lemma}
\begin{proof}
According to the definition of $w_{t,i}$, we have that
\begin{equation}
    w_{t,i} = \frac{1}{\| g_{t,i} \| + r / (\| g_{t,i} \| + r)} \leq \frac{1}{2 \sqrt{r} - r}.
\end{equation}
 Therefore, for the case that $- \tau_0 + (1-\tau_1)\|\nabla f(x)\|\leq \sqrt{r}-r$, we have
\begin{equation}
(L_0+L_1\|\nabla f(x_t)\|)w_{t,i} \leq \frac{L_0+L_1\frac{\sqrt{r}-r+\tau_0}{1-\tau_1}}{2\sqrt{r}-r} = \frac{L_0(1-\tau_1)+L_1(\sqrt{r}-r+\tau_0)}{(1-\tau_1)(2\sqrt{r}-r)}.
\end{equation}
Oppositely, for the case that $-\tau_0 + (1-\tau_1)\|\nabla f(x)\|> \sqrt{r}-r$, we have
\begin{equation}\label{bound_g}
     \| g_{t,i} \|  \geq -\| g_{t,i} - \nabla f(x) \| + \| \nabla f(x) \| \geq - \tau_0 + ( 1- \tau_1) \| \nabla f(x) \| > \sqrt{r} - r,
\end{equation}
where the first inequality and the second inequality follow from the Triangle inequality and Assumption 2, respectively.
Based on (\ref{bound_g}) and the monotonicity of $h(s) = \frac{1}{s + \frac{r}{s+r}}$ when $s > \sqrt{r}-r$, We have
\begin{equation}
    (L_0+L_1\|\nabla f(x_t)\|)w_{t,i} \leq \frac{(L_0+L_1\|\nabla f(x_t)\|)}{-\tau_0+(1-\tau_1)\|\nabla f(x)\|+\frac{r}{-\tau_0+r+(1-\tau_1)\|\nabla f(x)\|} }.
\end{equation}
We next find $K$ such that the following inequality is satisfied:
{\small
\begin{eqnarray}
\frac{(L_0+L_1\|\nabla f(x_t)\|)}{-\tau_0+(1-\tau_1)\|\nabla f(x)\|+\frac{r}{-\tau_0+r+(1-\tau_1)\|\nabla f(x)\|} } &\leq& K \frac{L_1}{1-\tau_1} \label{first_eqality_change_before}\\
 \iff (K-1)L_1[(1-\tau_1)\|\nabla f(x)\| - \tau_0 + r] + \frac{KL_1r}{(1-\tau_1)\|\nabla f(x)\| - \tau_0 + r} &\geq& \tau_0L_1 + (K-1)L_1r + L_0(1-\tau_1).\label{first_eqality_change}
  \end{eqnarray}
} 
As the lower bound of the left hand side of (\ref{first_eqality_change}) is $2\sqrt{(K-1)K L_1^2 r}$, to achieve (\ref{first_eqality_change}), we need
\begin{equation}\label{first_eqality_change2}
    2\sqrt{(K-1)K}L_1\sqrt{r} - (K-1)L_1r \geq \tau_0L_1 + L_0(1-\tau_1)
\end{equation}
Since $K > K-1$ and $0< r < 1$, we have
\begin{equation}
    2\sqrt{(K-1)K}L_1\sqrt{r} - (K-1)L_1r > \sqrt{(K-1)K}L_1\sqrt{r} > (K-1)L_1\sqrt{r}.
\end{equation}
Thus, to achieve (\ref{first_eqality_change2}), we need
 \begin{eqnarray}
  (K-1)L_1\sqrt{r} &\geq& \tau_0L_1 + L_0(1-\tau_1) \nonumber\\
 \iff K &\geq& \frac{L_0(1-\tau_1) + L_1\tau_0 + L_1\sqrt{r}}{\sqrt{r}L_1}.
\end{eqnarray}
Finally, setting $K = \frac{L_0(1-\tau_1) + L_1\tau_0 + L_1\sqrt{r}}{\sqrt{r}L_1}$, we have that (\ref{first_eqality_change_before}) holds. Therefore, we get
\begin{equation}
    (L_0 + L_1 \| \nabla f(x_t) \| ) w_{t,i} \leq \max \left( \frac{L_0(1-\tau_1)+L_1(\sqrt{r}-r+\tau_0)}{(1-\tau_1)(2\sqrt{r}-r)}, \frac{L_0(1-\tau_1) + L_1\tau_0 + L_1\sqrt{r}}{\sqrt{r}(1-\tau_1)}\right),
\end{equation}
which completes the proof.
\end{proof}

~~~

~~~

~~~


\begin{lemma}\label{lemma:eta_bound}
Under Assumption~2, for any $\alpha \in (0,1), i \in \{1, \cdots, b\}$,
if
$$\eta_t \leq \min \left(\frac{(2\sqrt{r}-r)(1-\tau_1)\alpha}{4\left(L_0(1-\tau_1)+L_1(\sqrt{r}-r+\tau_0)\right)}, \frac{(1-\tau_1)\sqrt{r}\alpha}{4\left(L_0(1-\tau_1)+L_1\tau_0 + L_1\sqrt{r}\right)},\frac{\alpha B^2}{6L_1d\sigma^2}\right),$$
then we have
\begin{eqnarray} \label{lemma:eta_bound:origin}
\frac{L_0+L_1\|\nabla f(x_t)\|}{2B^2}\eta_t^2d\sigma^2 &\leq& \frac{L_0+L_1(1+\tau_0)}{2B^2}\eta_t^2d\sigma^2 + \frac{\alpha\eta_t w_{t,i}}{4}\|\nabla f(x_t)\|^2, \nonumber\\
\left(L_0+L_1\|\nabla f(x_t)\|\right)\eta_t^2w_{t,i}^2 \left \langle\nabla f(x_t) , g_t- \nabla f(x_t) \right\rangle &\leq& \frac{(L_0(1-\tau_1)+L_1\tau_0)\tau_0^2\eta_t^2}{r^2(1-\tau_1)^3} + \frac{\alpha\eta_t w_{t,i}}{4}\|\nabla f(x_t)\|^2, \nonumber\\
\frac{L_0+L_1\|\nabla f(x_t)\|}{2}\eta_t^2w_{t,i}^2\|g_{t,i} - \nabla f(x_t)\|^2 &\leq& \frac{(L_0(1-\tau_1)+L_1\tau_0)\tau_0^2\eta_t^2}{2r^2(1-\tau_1)^3} + \frac{\alpha\eta_t w_{t,i}}{4}\|\nabla f(x_t)\|^2, \nonumber\\
\frac{L_0+L_1\|\nabla f(x_t)\|}{2}\eta_t^2w_{t,i}^2\|\nabla f(x_t)\|^2 &\leq& \frac{\alpha\eta_t w_{t,i}}{4}\|\nabla f(x_t)\|^2.
\end{eqnarray}
\end{lemma}

\begin{proof}
\noindent\textbf{The first formula.} For the first inequality in (\ref{lemma:eta_bound:origin}), we consider its upper bound from two cases. For $\| \nabla f(x_t)\| < 1 + \tau_0$, we have:
\begin{equation}
\frac{L_1\|\nabla f(x_t)\|}{2B^2}\eta_t^2d\sigma^2 \leq \frac{L_1(1+\tau_0)}{2B^2}\eta_t^2d\sigma^2.    
\end{equation}
And for $\| \nabla f(x_t)\| \geq 1 + \tau_0$, the following inequality holds: 
\begin{equation} \label{lemma:eta_bound:bound_w}
w_{t,i} = \frac{1}{\|g_{t,i}\|+\frac{r}{\|g_{t,i}\| + r}} \geq \frac{1}{\tau_0+(\tau_1+1)\|\nabla f(x_t)\| + 1} \geq \frac{1}{3\|\nabla f(x_t)\|},
\end{equation}
where the first inequality holds because $\| g_{t,i} \| \leq \tau_0 + (1+\tau_1) \nabla f(x)$ and $\frac{r}{\| g_{t,i} \| + r} \leq 1$, and the second inequality holds because $\tau_1 < 1$. According to (\ref{lemma:eta_bound:bound_w}), we set $ \eta_t \leq \frac{\alpha B^2}{6L_1d\sigma^2}$ such that
\begin{equation}
\frac{L_1\|\nabla f(x_t)\|}{2B^2}\eta_t^2d\sigma^2 \leq \frac{\alpha\eta_t w_{t,i}}{4}\|\nabla f(x_t)\|^2.
\end{equation}
At this point, the first formula is established.

\noindent\textbf{The fourth formula.} For the fourth formula in (\ref{lemma:eta_bound:origin}), for $\eta_t \leq \frac{\alpha}{2}min(\frac{(2\sqrt{r}-r)(1-\tau_1)}{L_0(1-\tau_1)+L_1(\sqrt{r}-r+\tau_0)}, \frac{(1-\tau_1)\sqrt{r}}{[L_0(1-\tau_1)+L_1\tau_0 + L_1\sqrt{r}]})$, it can be directly obtained from Lemma 2 that
\begin{equation}
\frac{L_0+L_1\|\nabla f(x_t)\|}{2}\eta_t^2w_{t,i}^2\|\nabla f(x_t)\|^2 \leq \frac{\alpha\eta_t w_{t,i}}{4}\|\nabla f(x_t)\|^2.
\end{equation}

\noindent\textbf{The second and the third formulas.}
For the second and the third formulas in (\ref{lemma:eta_bound:origin}), they are also divided into two cases to consider. For the case that $\|\nabla f(x_t)\| \leq \tau_0/(1-\tau_1)$, we have
\begin{eqnarray}
(L_0 + L_1\|\nabla f(x_t)\|)w_{t,i}^2\langle \nabla f(x_t), g_{t,i} - \nabla f(x_t)\rangle 
&\leq& (L_0 + L_1\|\nabla f(x_t)\|)w_{t,i}^2 \| \nabla f(x_t)\| \|g_{t,i} - \nabla f(x_t)\| \nonumber\\
&\leq& (L_0 + L_1\|\nabla f(x_t)\|)w_{t,i}^2 \| \nabla f(x_t)\| (\tau_0 + \tau_1 \| \nabla f(x_t) \|) \nonumber\\
&\leq& \frac{\tau_0^2(L_0(1-\tau_1)+L_1\tau_0)}{r^2(1-\tau_1)^3}
\end{eqnarray}
and
\begin{eqnarray}
\frac{L_0 + L_1\|\nabla f(x_t)\|}{2}w_{t,i}^2 \| g_{t,i} - \nabla f(x_t) \|^2 \leq \frac{(L_0 + L_1\|\nabla f(x_t)\|)}{2}w_{t,i}^2  (\tau_0 + \tau_1 \| \nabla f(x_t) \|)^2
\leq  \frac{\tau_0^2(L_0(1-\tau_1)+L_1\tau_0)}{2r^2(1-\tau_1)^3}.
\end{eqnarray}
For the case that $\|\nabla f(x_t)\| > \tau_0/(1-\tau_1)$, recalling the result in Lemma~2 and setting $\eta_t \leq \frac{\alpha}{4} \min \left(\frac{(2\sqrt{r}-r)(1-\tau_1)}{L_0(1-\tau_1)+L_1(\sqrt{r}-r+\tau_0)}, \frac{(1-\tau_1)\sqrt{r}}{L_0(1-\tau_1)+L_1\tau_0 + L_1\sqrt{r}} \right)$, we can observe the following inequalities:
\begin{eqnarray}\nonumber
(L_0 + L_1\|\nabla f(x_t)\|)\eta_t^2w_{t,i}^2\langle \nabla f(x_t), g_{t,i} - \nabla f(x_t)\rangle &\leq& (L_0 + L_1\|\nabla f(x_t)\|)\eta_t^2w_{t,i}^2 \|\nabla f(x_t)\| (\tau_0 + \tau_1\|\nabla f(x_t)\|) \nonumber\\ 
&\leq& (1-\tau_1)(L_0 + L_1\|\nabla f(x_t)\|)\eta_t^2w_{t,i}^2 \|\nabla f(x_t)\|^2 \nonumber\\
&\leq& \frac{\alpha\eta_t w_t}{4}\|\nabla f(x_t)\|^2
\end{eqnarray}
and 
\begin{eqnarray}
\frac{L_0 + L_1\|\nabla f(x_t)\|}{2}\eta_t^2w_{t,i}^2\|g_{t,i} - \nabla f(x_t)\|^2 \leq \frac{\alpha\eta_t w_{t,i}}{4}\|\nabla f(x_t)\|^2.
\end{eqnarray}
Combining the above two cases, Lemma 3 is proved.
\end{proof}

~~~

~~~

~~~

\begin{lemma}
    Under Assumption~2, For any $0 < r \leq 1, i \in \{1, \cdots, b\}$, 
    the following inequality holds:
    \begin{equation}
    w_{t,i}\|\nabla f(x_t)\| \leq \max \left(\frac{\sqrt{r}-r+\tau_0}{(2\sqrt{r}-r)(1-\tau_1)}, \frac{\tau_0+\sqrt{r}}{(1-\tau_1)\sqrt{r}}\right).
    \end{equation}
\end{lemma}

\begin{proof}
For the case that $(1-\tau_1)\|\nabla f(x_t)\| -\tau_0 < \sqrt{r}-r$, we have
\begin{eqnarray}
w_{t,i}\|\nabla f(x_t)\| = \frac{\|\nabla f(x_t)\|}{\|g_{t,i}\| + \frac{r}{\|g_{t,i}\|+r}} 
\leq \frac{\|\nabla f(x_t)\|}{2\sqrt{r}-r} 
\leq \frac{\sqrt{r}-r+\tau_0}{(2\sqrt{r}-r)(1-\tau_1)}.
\end{eqnarray}
For the case that $(1-\tau_1)\|\nabla f(x_t)\| -\tau_0 \geq \sqrt{r}-r$, we look for a constant $K$ such that $K$ satisfies:
\begin{eqnarray}
    w_{t,i}\|\nabla f(x_t)\| &=& \frac{\|\nabla f(x_t)\|}{\|g_{t,i}\| + \frac{r}{\|g_{t,i}\|+r}} 
    \leq \frac{\|\nabla f(x_t)\|}{(1-\tau_1)\|\nabla f(x_t)\|-\tau_0 + \frac{r}{(1-\tau_1)\|\nabla f(x_t)\|-\tau_0 + r}} 
    \leq K,
\end{eqnarray}
where the first inequality comes from (\ref{bound_g}).
The last inequality above can be transformed into:
\begin{eqnarray}
&& \|\nabla f(x_t)\| \leq K(1-\tau_1)\|\nabla f(x_t)\| - K\tau_0 + \frac{rK}{r-\tau_0 + (1-\tau_1)\|\nabla f(x_t)\|}\nonumber\\
&\iff& (K-\frac{1}{1-\tau_1})(1-\tau_1)\|\nabla f(x_t)\|  - (K-\frac{1}{1-\tau_1})\tau_0 +  (K-\frac{1}{1-\tau_1})r+ \frac{rK}{r-\tau_0 + (1-\tau_1)\|\nabla f(x_t)\|} \nonumber\\
&&\geq (K-\frac{1}{1-\tau_1})r+\frac{\tau_0}{1-\tau_1}. \label{second_equality_change}
\end{eqnarray}
Since the L.H.S. of (\ref{second_equality_change}) can be bounded as
\begin{eqnarray}
&&(K-\frac{1}{1-\tau_1})(1-\tau_1)\|\nabla f(x_t)\|  - (K-\frac{1}{1-\tau_1})\tau_0 +  (K-\frac{1}{1-\tau_1})r+ \frac{rK}{r-\tau_0 + (1-\tau_1)\|\nabla f(x_t)\|} \nonumber\\
&\geq& 2\sqrt{(K-\frac{1}{1-\tau_1})rK} \geq 2(K-\frac{1}{1-\tau_1})\sqrt{r}
\end{eqnarray}
and $r<1$, to achieve (\ref{second_equality_change}), we need
\begin{eqnarray}
&& 2(K-\frac{1}{1-\tau_1})\sqrt{r} \geq (K-\frac{1}{1-\tau_1})r+\frac{\tau_0}{1-\tau_1}\\\nonumber
&\Longleftarrow& (K-\frac{1}{1-\tau_1})\sqrt{r} \geq  \frac{\tau_0}{1-\tau_1}\\\nonumber
&\iff& K \geq \frac{\tau_0+\sqrt{r}}{(1-\tau_1)\sqrt{r}}.
\end{eqnarray}
We set $K = \frac{\tau_0+\sqrt{r}}{(1-\tau_1)\sqrt{r}} $, and combine with the upper bound of the first case, then we complete the proof.
\end{proof}

~~~

~~~

~~~

\begin{lemma}\label{lemma:scale}
Under Assumption~2,
For $\alpha = N(\tau_0,\tau_1,r)(1-\tau_1)\min \left(\frac{\sqrt{r}}{\sqrt{r}+\tau_0},\frac{2\sqrt{r}-r}{\sqrt{r}-r+\tau_0}\right)/4$, the following inequality holds:
\begin{eqnarray}
&&- \eta_t \mathbb{E}_t \langle w_{t,i} \nabla f(x_t), g_{t,i}\rangle + \alpha \eta_t \mathbb{E}_t w_{t,i} \|\nabla f(x_t)\|^2 \nonumber\\
&\leq& \left\{
\begin{aligned}
&- \frac{3\eta_t N(\tau_0,\tau_1,r)}{4}\|\nabla f(x_t)\|  &\| \nabla f(x_t)\| \geq \frac{\tau_0}{(1-\tau_1)}, \nonumber\\
&- \frac{7\eta_t M(\tau_0,\tau_1,r)}{8}\|\nabla f(x_k)\|^2 +  \frac{2 \eta_t \tau_0^3}{(1-\tau_1)^3(\tau_0 + \frac{r(1-\tau_1)}{2\tau_0 + r(1-\tau_1)})(2\sqrt{r}-r)} &\| \nabla f(x_t)\| < \frac{\tau_0}{(1-\tau_1)},
\end{aligned}
\right.
\end{eqnarray}
where
\begin{equation}
N(\tau_0, \tau_1, r) = \min \left(\frac{\tau_0}{1-\tau_1}, \frac{2\tau_0^2 + r\tau_0(1-\tau_1)}{4\tau_0^2 + 2r\tau_0(1-\tau_1)+r(1-\tau_1)^2}\right)\nonumber
\end{equation}
and
\begin{equation}
M(\tau_0, \tau_1, r) = \min \left(1, \frac{1}{\tau_0 + \frac{r}{r+\tau_0} + \frac{\tau_0 (1 + \tau_1)}{1-\tau_1}}\right). \nonumber
\end{equation}
\end{lemma}

\begin{proof}
When $\| \nabla f(x_t)\| \geq \frac{\tau_0}{(1-\tau_1)}$ holds, we can establish the following inequality:
\begin{eqnarray}
\langle\nabla f(x_t), g_{t,i}\rangle &=& \|\nabla f(x_t)\|^2 + \langle\nabla f(x_t), g_{t,i} - \nabla f(x_t)\rangle \nonumber\\
&\geq& \|\nabla f(x_t)\|^2 - \|\nabla f(x_t)\|  \| g_{t,i} - \nabla f(x_t)\| \nonumber\\
&\geq& (1-\tau_1)\|\nabla f(x_t)\|^2 -\tau_0\|\nabla f(x_t)\| \geq 0.
\end{eqnarray}
According to Assumption~2, we have $-\tau_0 + (1-\tau_1 \nabla f(x_t)) \leq \| g_{t,i} \| \leq \tau_0 + (1+\tau_1) \nabla f(x_t) \|$. Then, we obtain
\begin{eqnarray}\nonumber
\|g_{t,i}\| + \frac{r}{\|g_{t,i}\| + r} &\leq& \max \left((1-\tau_1)\|\nabla f(x_t)\| - \tau_0 + 1, (1+\tau_1)\|\nabla f(x_t)\| + \tau_0 + \frac{r}{(1+\tau_1)||\nabla f(x_t)|| + \tau_0 + r}\right)\\\nonumber
&\leq& \max \left((1-\tau_1)\|\nabla f(x_t)\| - \tau_0 + 1, (1+\tau_1)\|\nabla f(x_t)\| + \tau_0 + \frac{r}{\frac{1+\tau_1}{1-\tau_1}\tau_0 + \tau_0 + r}\right)\\\nonumber
&\leq& \max \left((1-\tau_1)\|\nabla f(x_t)\| + (1 - \tau_0) \frac{1-\tau_1}{\tau_0}\|\nabla f(x_t)\|, (1+\tau_1)\|\nabla f(x_t)\| + \tau_0 + \frac{r(1-\tau_1)}{2\tau_0 + r(1-\tau_1)}\right)\\\nonumber
&\leq& \max \left(\frac{1-\tau_1}{\tau_0}\|\nabla f(x_t)\|, (2+\frac{r(1-\tau_1)^2}{2\tau_0^2 + r\tau_0(1-\tau_1)})\|\nabla f(x_t)\|\right).
\end{eqnarray}
Setting $N(\tau_0, \tau_1, r) = \min \left(\frac{\tau_0}{1-\tau_1}, \frac{2\tau_0^2 + r\tau_0(1-\tau_1)}{4\tau_0^2 + 2r\tau_0(1-\tau_1)+r(1-\tau_1)^2}\right),$
the following inequalities hold:
\begin{eqnarray}
    \mathbb{E}_t w_{t,i}\langle\nabla f(x_t), g_{t,i}\rangle   \geq \mathbb{E}_t \frac{\langle\nabla f(x_t), g_t\rangle}{\|g_{t,i}\|+ \frac{r}{\|g_{t,i}\|+r}}
    \geq \mathbb{E}_t N(\tau_0, \tau_1,r)\frac{\langle\nabla f(x_t), g_{t,i}\rangle}{\|\nabla f(x_t)\|}
    = N(\tau_0, \tau_1,r)\|\nabla f(x_t)\|,
\end{eqnarray}
\begin{equation}
\alpha = N(\tau_0,\tau_1,r)(1-\tau_1)\min (\frac{\sqrt{r}}{\sqrt{r}+\tau_0},\frac{2\sqrt{r}-r}{\sqrt{r}-r+\tau_0})/4 < 1/8.
\end{equation}
Meanwhile, from Lemma 4 we obtain that
\begin{equation} \alpha w_{t,i}\|\nabla f(x_t)\|^2 \leq  \max (\frac{\sqrt{r}-r+\tau_0}{(2\sqrt{r}-r)(1-\tau_1)}, \frac{\tau_0+\sqrt{r}}{(1-\tau_1)\sqrt{r}})\alpha \|\nabla f(x_t)\| 
= \frac{N(\tau_0,\tau_1,r)}{4}\|\nabla f(x_t)\|.
\end{equation}
Then we get 
\begin{equation}
- \eta_t \mathbb{E}_t \langle w_{t,i} \nabla f(x_t), g_{t,i}\rangle + \alpha \eta_t \mathbb{E}_t w_{t,i} \|\nabla f(x_t)\|^2 
 \leq - \frac{3\eta_t N(\tau_0,\tau_1,r)}{4}\|\nabla f(x_t)\|.
\end{equation}

~~

~~

On the contrary, when $\| \nabla f(x_t)\| \leq \frac{\tau_0}{(1-\tau_1)}$ is established, 
defining $Q = \tau_0 + (1+\tau_1) \| \nabla f(x_t)\|$, we have
\begin{eqnarray}
    &&\mathbb{E}_t w_{t,i}\langle\nabla f(x_t), g_{t,i}\rangle - \alpha \mathbb{E}_t w_{t,i} \|\nabla f(x_t)\|^2  \nonumber\\
    &=&\mathbb{E}_t\frac{\langle\nabla f(x_t), g_t\rangle}{\|g_{t,i}\|+ \frac{r}{\|g_{t,i}\|+r}}- \alpha \mathbb{E}_t w_{t,i} \|\nabla f(x_t)\|^2 \nonumber\\
    &=& \mathbb{E}_t\frac{\langle\nabla f(x_t), \nabla f(x_t)\rangle}{\|g_{t,i}\|+ \frac{r}{\|g_{t,i}\|+r}}+ E_t\frac{\langle\nabla f(x_t), g_{t,i}-\nabla f(x_t)\rangle}{\|g_{t,i}\|+ \frac{r}{\|g_{t,i}\|+r}}- \alpha \mathbb{E}_t w_t \|\nabla f(x_t)\|^2 \nonumber\\
    & \geq& \frac{7}{8}\mathbb{E}_t\frac{\|\nabla f(x_t)\|^2}{\|g_{t,i}\|+ \frac{r}{\|g_{t,i}\|+r}} + \mathbb{E}_t \frac{\langle\nabla f(x_t), g_{t,i}-\nabla f(x_t)\rangle}{Q+ \frac{r}{Q+r}} \nonumber\\
    && ~~~~~~~~~+ \mathbb{E}_t\frac{(\langle\nabla f(x_t), g_{t,i}-\nabla f(x_t)\rangle)(Q +\frac{r}{Q+r} - \|g_{t,i}\|- \frac{r}{\|g_{t,i}\|+r} )}{(\|g_{t,i}\|+ \frac{r}{\|g_{t,i}\|+r})(Q +\frac{r}{Q+r})} \nonumber\\
    &=& \frac{7}{8}\mathbb{E}_t \frac{\|\nabla f(x_t)\|^2}{\|g_{t,i}\|+ \frac{r}{\|g_{t,i}\|+r}}+ \mathbb{E}_t\frac{(\langle\nabla f(x_t), g_{t,i}-\nabla f(x_t)\rangle)(Q +\frac{r}{Q+r} - \|g_{t,i}\|- \frac{r}{\|g_{t,i}\|+r} )}{(\|g_{t,i}\|+ \frac{r}{\|g_{t,i}\|+r})(Q +\frac{r}{Q+r})} \nonumber\\
    &\geq& \frac{7}{8}\mathbb{E}_t\frac{\|\nabla f(x_t)\|^2}{\|g_{t,i}\|+ \frac{r}{\|g_{t,i}\|+r}} -\frac{2\tau_0^2(1+\tau_0)}{(1-\tau_1)^3(\tau_0 + \frac{r(1-\tau_1)}{2\tau_0 + r(1-\tau_1)})(2\sqrt{r}-r)}.
\end{eqnarray}
The last inequality holds because
\begin{eqnarray}
    &&\mathbb{E}_t \frac{(\langle\nabla f(x_t), g_{t,i}-\nabla f(x_t)\rangle)(Q +\frac{r}{Q+r} - \|g_{t,i}\|- \frac{r}{\|g_{t,i}\|+r} )}{(\|g_{t,i}\|+ \frac{r}{\|g_{t,i}\|+r})(Q +\frac{r}{Q+r})} \nonumber\\
    & = &\mathbb{E}_t \frac{(Q-\|g_{t,i}\|)(1-\frac{r}{(\|g_{t,i}\|+r)(Q+r)})(\langle\nabla f(x_t), g_{t,i}-\nabla f(x_t)\rangle)}{(\|g_{t,i}\|+ \frac{r}{\|g_{t,i}\|+r})(Q +\frac{r}{Q+r})} \nonumber\\
    & \geq& -\mathbb{E}_t \frac{(Q-\|g_{t,i}\|)|(1-\frac{r}{(\|g_{t,i}\|+r)(Q+r)})|\|\nabla f(x_t)\|(\tau_0 + \tau_1 \| \nabla f(x_t)\|)}{(\|g_{t,i}\|+ \frac{r}{\|g_{t,i}\|+r})(Q +\frac{r}{Q+r})}\nonumber\\
    &\geq & -\mathbb{E}_t \frac{(Q-\|g_{t,i}\|)(1+\frac{r}{(\|g_{t,i}\|+r)(Q+r)})\|\nabla f(x_t)\|(\tau_0 + \tau_1 \| \nabla f(x_t)\|)}{(\|g_{t,i}\|+ \frac{r}{\|g_{t,i}\|+r})(Q +\frac{r}{Q+r})}\nonumber\\
    &\geq&  -\mathbb{E}_t \frac{(Q-\|g_{t,i}\|)(1+\frac{1}{Q})\|\nabla f(x_t)\|(\tau_0 + \tau_1 \| \nabla f(x_t)\|)}{(\|g_{t,i}\|+ \frac{r}{\|g_{t,i}\|+r})(Q +\frac{r}{Q+r})}\nonumber\\
    &\geq&  -\mathbb{E}_t \frac{(Q-\|g_{t,i}\|)(1+\frac{1}{\tau_0})\|\nabla f(x_t)\|(\tau_0 + \tau_1 \| \nabla f(x_t)\|)}{(\|g_{t,i}\|+ \frac{r}{\|g_{t,i}\|+r})(Q +\frac{r}{Q+r})}\nonumber\\
    & \geq& -\mathbb{E}_t \frac{(Q-\|g_{t,i}\|)(1+\frac{1}{\tau_0})\|\nabla f(x_t)\|(\tau_0 + \tau_1 \| \nabla f(x_t)\|)}{(Q+\frac{r}{Q+r})(2\sqrt{r}-r)} \nonumber\\
    & \geq &  -\mathbb{E}_t \frac{(Q-\|g_{t,i}\|)(1+\frac{1}{\tau_0})\|\nabla f(x_t)\|(\tau_0 + \tau_1 \| \nabla f(x_t)\|)}{(\tau_0+\frac{r}{2\tau_0/(1-\tau_1)+r})(2\sqrt{r}-r)}\nonumber\\
    & \geq& -\frac{2\tau_0/(1-\tau_1)(1+\frac{1}{\tau_0})\|\nabla f(x_t)\|(\tau_0 + \tau_1 \| \nabla f(x_t)\|)}{(\tau_0 + \frac{r(1-\tau_1)}{2\tau_0 + r(1-\tau_1)})(2\sqrt{r}-r)} \nonumber\\
    & \geq& -\frac{2\tau_0^2(1+\tau_0)}{(1-\tau_1)^3(\tau_0 + \frac{r(1-\tau_1)}{2\tau_0 + r(1-\tau_1)})(2\sqrt{r}-r)},
\end{eqnarray}
where the above inequalities hold because $\|  \nabla f(x_t)\| \leq \frac{\tau_0}{1-\tau_1}$ and $\tau_0 \leq Q \leq \frac{2\tau_0}{1-\tau_1}.$
Meanwhile, we know that
\begin{eqnarray}
    \|g_{t,i}\| + \frac{r}{\|g_{t,i}\| + r} &\leq& \max \left(1, (1+\tau_1)\|\nabla f(x_t)\|+\tau_0  +\frac{r}{(1+\tau_1)\|\nabla f(x_t)\|+r+\tau_0} \right) \nonumber\\
    &\leq& \max \left(1,\tau_0 + (1+\tau_1)\frac{\tau_0}{1-\tau_1} + \frac{r}{r+\tau_0}\right).
\end{eqnarray}  
Setting $M(\tau_0, \tau_1, r) = \min \left(1, \frac{1}{\tau_0 + \frac{r}{r+\tau_0} + \frac{\tau_0 (1 + \tau_1)}{1-\tau_1}}\right)$ gives
\begin{eqnarray}
&&- \eta_t \mathbb{E}_t \langle w_{t,i} \nabla f(x_t), g_{t,i}\rangle + \alpha \eta_t \mathbb{E}_t w_{t,i} \|\nabla f(x_t)\|^2 \nonumber\\
&\leq& -\frac{7\eta_t M(\tau_0,\tau_1,r)}{8}\|\nabla f(x_k)\|^2 + \frac{2 \eta_t \tau_0^2(1+\tau_0)}{(1-\tau_1)^3(\tau_0 + \frac{r(1-\tau_1)}{2\tau_0 + r(1-\tau_1)})(2\sqrt{r}-r)}.
\end{eqnarray}  
Combining the above two situations, Lemma 5 is proved.
\end{proof}

~~~

\section{Detailed Proofs for Results in the Main Paper.}

\subsection{The proof of Theorem~1}
\begin{proof}
DP-PSAC and DP-SGD rely on a Gaussian mechanism to achieve differential privacy in each iteration. And the privacy promise of Gaussian mechanism relies on $l_2$ sensitivity and noise multiplier. While DP-PSAC can provide $l_2$ sensitivity consistent with DP-SGD. Therefore, under the same parameter setting, Theorem~1 can be obtained from Lemma A.1 \cite{abadi2016deep}.

\end{proof}
\subsection{The proof of Theorem~2}
\begin{proof}
At first,
we transform Lemma 1 into the following equivalent formula:
\begin{eqnarray} 
\mathbb{E}_t(f(x_{t+1})) - f(x_t) &\leq& \frac{1}{B}\sum\limits_{i=1}^B \left(- \eta_t \mathbb{E}_t\langle w_{t,i} \nabla f(x_t), g_{t,i}\rangle \right) \nonumber\\
&& + \frac{1}{B}\sum\limits_{i=1}^B \left(\frac{L_0+L_1\|\nabla f(x)\|}{2}\eta_t^2 \left(\mathbb{E}_t w_t^2 \|\nabla f(x_t)\|^2 + 2\mathbb{E}w_{t,i}^2\langle g_{t,i} - \nabla f(x_t), \nabla f(x_t)\rangle \right)\right) \nonumber\\
&& + \frac{1}{B}\sum\limits_{i=1}^B \left(\frac{L_0+L_1\|\nabla f(x)\|}{2}\eta_t^2 \left(\frac{d\sigma^2}{B^2}+ \mathbb{E}w_{t,i}^2\|g_{t,i}-\nabla f(x_t)\|^2 \right)\right).
\end{eqnarray}
According to Lemma 3, we have
\begin{eqnarray}
 \mathbb{E}_t(f(x_{t+1})) - f(x_t) &\leq& - \frac{1}{B}\sum\limits_{i=1}^B \left(\eta_t \mathbb{E}_t\langle w_{t,i} \nabla f(x_t), g_{t,i}\rangle + \alpha \eta_t \mathbb{E}_t w_{t,i} \|\nabla f(x_t)\|^2 \right) \nonumber\\
 && + \eta_t^2 \left(\frac{3\tau_0^2(L_0(1-\tau_1)+L_1\tau_0)}{2r^2(1-\tau_1)^3} + \frac{L_1(1+\tau_0)+L_0}{2B^2}d\sigma^2 \right).
\end{eqnarray}
Setting $\eta_t = \eta$, and combining the above inequality with Lemma 5, 
we get that for $\| \nabla f(x_t)\| \geq \frac{\tau_0}{(1-\tau_1)}$:
\begin{equation}
 \frac{3\eta N(\tau_0, \tau_1, r)}{4}\|\nabla f(x_t)\| \leq f(x_t) - \mathbb{E}_t(f(x_{t+1})) +  \eta^2 \left(\frac{3\tau_0^2(L_0(1-\tau_1)+L_1\tau_0)}{2r^2(1-\tau_1)^3} + \frac{L_1(1+\tau_0)}{2B^2}d\sigma^2\right),  
\end{equation}
and for $\| \nabla f(x_t)\| < \frac{\tau_0}{(1-\tau_1)}$: 
\begin{eqnarray}
&&\eta \left(\frac{7M(\tau_0, \tau_1, r)}{8}\|\nabla f(x_t)\|^2 -\frac{2\tau_0^2(1+\tau_0)}{(1-\tau_1)^3(\tau_0 + \frac{r(1-\tau_1)}{2\tau_0 + r(1-\tau_1)})(2\sqrt{r}-r)}\right)\nonumber\\
&\leq& f(x_t) - \mathbb{E}_t(f(x_{t+1})) +  \eta^2 \left(\frac{3\tau_0^2(L_0(1-\tau_1)+L_1\tau_0)}{2r^2(1-\tau_1)^3} + \frac{L_1(1+\tau_0)}{2B^2}d\sigma^2\right). 
\end{eqnarray}
We divide the entire iteration into two parts, setting $U = \{t|t<T, \| \nabla f(x_t)\| \geq \frac{\tau_0}{(1-\tau_1)}\}$ and $U_c = \{0,\cdots,T-1\}/U$ then
\begin{eqnarray}
&&\max \left(\frac{3N(\tau_0, \tau_1, r)}{4|U|}\sum\limits_{t \in U}\|\nabla f(x_k)\|,\frac{7M(\tau_0, \tau_1, r)}{8|U_c|}\sum\limits_{t \in U_c}\|\nabla f(x_t)\|^2 \right)\nonumber\\
&\leq& \frac{D_f}{T\eta} + \eta\frac{(L_1(1+\tau_0)+L_1)d\sigma^2}{2B^2}+ \eta\frac{3\tau_0^2(L_0(1-\tau_1)+L_1\tau_0)}{2r^2(1-\tau_1)^3}+ \frac{2\tau_0^2(1+\tau_0)}{(1-\tau_1)^3(\tau_0 + \frac{r(1-\tau_1)}{2\tau_0 + r(1-\tau_1)})(2\sqrt{r}-r)} \frac{|U_c|}{T},
\end{eqnarray}
where
\begin{equation}
  D_f = f(x_0) - \mathbb{E}_{T-1}[f(x_T)].  
\end{equation}
Setting
\begin{equation}
  \eta = \sqrt{\frac{2B^2}{d\sigma^2T(L_0+L_1(\tau_0+1))}}, 
\end{equation}
we get
\begin{eqnarray}
&&\max \left(\frac{3N(\tau_0, \tau_1, r)}{4|U|}\sum\limits_{t \in U}\|\nabla f(x_k)\|,\frac{7M(\tau_0, \tau_1, r)}{8|U_c|}\sum\limits_{t \in U_c}\|\nabla f(x_t)\|^2\right)\nonumber\\
&\leq& (D_f + 1)\sqrt{\frac{d\sigma^2(L_0 + L_1 (\tau_0 +1))}{2TB^2}} \nonumber\\
&&+  \frac{3(L_0(1-\tau_1)+L_1 \tau_0)\tau_0^2}{2r^2(1-\tau_1)^3} \sqrt{\frac{(L_0 + L_1 (\tau_0+1))B^2}{2Td\sigma^2}}  + \frac{2\tau_0^2(1+\tau_0)}{(1-\tau_1)^3(\tau_0 + \frac{r(1-\tau_1)}{2\tau_0 + r(1-\tau_1)})(2\sqrt{r}-r)} \frac{|U_c|}{T}.
\end{eqnarray}
Defining
\begin{eqnarray}
  \Delta &=& (D_f + 1)\sqrt{\frac{d\sigma^2(L_0 + L_1 (\tau_0 +1))}{2TB^2}} \nonumber\\
&&+  \frac{3(L_0(1-\tau_1)+L_1 \tau_0)\tau_0^2}{2r^2(1-\tau_1)^3} \sqrt{\frac{(L_0 + L_1 (\tau_0+1))B^2}{2Td\sigma^2}}  + \frac{2\tau_0^2(1+\tau_0)}{(1-\tau_1)^3(\tau_0 + \frac{r(1-\tau_1)}{2\tau_0 + r(1-\tau_1)})(2\sqrt{r}-r)}, 
\end{eqnarray}
we obtain
\begin{eqnarray}
\mathbb{E}(\mathop{\min}\limits_{0 \leq t < T}\|\nabla f(x_t)\|) &\leq& \mathbb{E}[\min (\sqrt{\frac{1}{|U_c|}\sum\limits_{k \in U_c}\|\nabla f(x_t)\|^2},\frac{1}{|U|}\sum\limits_{k \in U}\|\nabla f(x_t)\|)] \nonumber\\
&\leq& \max(\sqrt{\frac{16}{7M(\tau_0, \tau_1,r)}\Delta}, \frac{8}{3N(\tau_0, \tau_1, r)}\Delta),
\end{eqnarray}
where the second inequality is because when $U \geq T/2$ the following inequality holds:
\begin{equation}
 \mathbb{E}\min \left(\sqrt{\frac{1}{|U_c|}\sum\limits_{k \in U_c}\|\nabla f(x_t)\|^2},\frac{1}{|U|}\sum\limits_{k \in U}\|\nabla f(x_t)\|\right) \leq \frac{2}{|U|}\sum\limits_{k \in U}\|\nabla f(x_t)\| \leq \frac{8}{3N(\tau_0, \tau_1, r)}\Delta, 
\end{equation}
and when $U < T/2, U_c \geq T/2$ the following inequality holds:
\begin{equation}
\mathbb{E}\min \left(\sqrt{\frac{1}{|U_c|}\sum\limits_{k \in U_c}\|\nabla f(x_t)\|^2},\frac{1}{|U|}\sum\limits_{k \in U}\|\nabla f(x_t)\|\right) \leq \sqrt{\frac{2}{|U_c|}\sum\limits_{k \in U_c}\|\nabla f(x_t)\|^2} \leq \sqrt{\frac{16}{7M(\tau_0, \tau_1,r)}\Delta}.
\end{equation}
Then, we consider the order of magnitude to get:
\begin{eqnarray}
&&\mathbb{E}\mathop{min}\limits_{0 \leq t < T}\|\nabla f(x_t)\| \nonumber\\
&\leq& \mathcal{O}(\sqrt{\frac{16(D_f + 1)}{7M(\tau_0, \tau_1, r)} \sqrt{\frac{d\sigma^2(L_0 + L_1 (\tau_0 +1))}{2TB^2}}} \nonumber\\
&&+ \sqrt{\frac{48(L_0(1-\tau_1)+L_1 \tau_0)\tau_0^2}{7M(\tau_0, \tau_1, r)r^2(1-\tau_1)^3} \sqrt{\frac{(L_0 + L_1 (\tau_0+1))B^2}{2Td\sigma^2}}}) +\frac{8\tau_0^2(1+\tau_0)}{3N(\tau_0, \tau_1, r)(1-\tau_1)^3(\tau_0 + \frac{r(1-\tau_1)}{2\tau_0 + r(1-\tau_1)})(2\sqrt{r}-r)} \nonumber\\
&=&\mathcal{O}(\sqrt[4]{\frac{d\sigma^2}{TB^2}} + \sqrt[4]{\frac{B^2}{Td\sigma^2}}) + \frac{8\tau_0^2(1+\tau_0)}{3N(\tau_0, \tau_1, r)(1-\tau_1)^3(\tau_0 + \frac{r(1-\tau_1)}{2\tau_0 + r(1-\tau_1)})(2\sqrt{r}-r)},
\end{eqnarray}
which completes the proof.
\end{proof}

~~
~~

Note that the premise of the above proof is $\eta = \sqrt{\frac{2B^2}{d\sigma^2T(L_0+L_1(\tau_0+1))}}$, and $\eta$ also needs to satisfy the constraints given in Lemma 3. 
All of the above proofs holds when we have the following Lemma.
\begin{lemma}\label{T_lowerbound}
When $\eta = \sqrt{\frac{2B^2}{d\sigma^2T(L_0+L_1(\tau_0+1))}}$, the condition in Lemma 3 is established only if $T$ is large enough to satisfy the following conditions:
\begin{eqnarray}\nonumber
T &\geq& A(L, \tau, d, r, \sigma, B) \\\nonumber
&=& \max \left(\frac{32B^2\left(L_0(1-\tau_1)+L_1(\sqrt{r}-r+\tau_0)\right)^2}{d\sigma^2(L_0+L_1(\tau_0+1))(2\sqrt{r}-r)^2(1-\tau_1)^2\alpha^2}, \frac{32B^2\left(L_0(1-\tau_1)+L_1\tau_0+L_1\sqrt{r}\right)^2}{d\sigma^2(L_0+L_1(\tau_0+1))(1-\tau_1)^2r\alpha^2},\frac{72L_1^2d\sigma^2}{(L_0+L_1(\tau_0+1))B^2}\right).
\end{eqnarray}
\end{lemma}
\begin{proof}
We can get this lemma by substituting $\eta = \sqrt{\frac{2B^2}{d\sigma^2T(L_0+L_1(\tau_0+1))}}$ into Lemma 3.
\end{proof}

~~

~~

~~

~~

\subsection{Lemmas for the proof of Corollary~1}

To achieve the $(\epsilon, \delta)$ DP guarantee, we can obtain from Theorem~2 that the noise multiplier $\sigma$ depends on the privacy parameters $(\epsilon, \delta)$  and the number of iterations $T$. 
By setting $\sigma$ properly, we can extend Theorem~2 to observe the following upper bound.

\begin{lemma}\label{lemma:converge2}
With the same setting as Theorem~2, 
to achieve $(\epsilon, \delta)$ differential privacy guarantees, the gradient norm can be bounded as:


\begin{eqnarray}
&&\mathbb{E}(\mathop{min}\limits_{0 \leq t < T}\|\nabla f(x_t)\|) \leq \mathcal{O}( \sqrt{\frac{ \sqrt{d\log(1/\delta)}}{N\epsilon}} + \sqrt{\frac{N\epsilon}{Td\sqrt{\log(1 /\delta)}}}) +\frac{8\tau_0^2(1+\tau_0)}{3N(\tau_0, \tau_1, r)(1-\tau_1)^3(\tau_0 + \frac{r(1-\tau_1)}{2\tau_0 + r(1-\tau_1)})(2\sqrt{r}-r)}. \nonumber
\end{eqnarray}

\end{lemma}

\begin{proof}
To achieve differential privacy commitment, we set $\sigma = c_2\frac{q\sqrt{T\log(1/\delta)}}{N\epsilon}$ and substitute it into (53):
\begin{eqnarray}
&&\mathbb{E}(\mathop{min}\limits_{0 \leq t < T}\|\nabla f(x_t)\|) \nonumber\\
&\leq&\mathcal{O}(\sqrt[4]{\frac{d\sigma^2}{TB^2}} + \sqrt[4]{\frac{B^2}{Td\sigma^2}}) + \frac{8\tau_0^2(1+\tau_0)}{3N(\tau_0, \tau_1, r)(1-\tau_1)^3(\tau_0 + \frac{r(1-\tau_1)}{2\tau_0 + r(1-\tau_1)})(2\sqrt{r}-r)}  \nonumber\\
&=&\mathcal{O}( \sqrt{\frac{ \sqrt{d\log(1/\delta)}}{N\epsilon}} + \sqrt{\frac{N\epsilon}{Td\sqrt{\log(1 /\delta)}}}) + \frac{8\tau_0^2(1+\tau_0)}{3N(\tau_0, \tau_1, r)(1-\tau_1)^3(\tau_0 + \frac{r(1-\tau_1)}{2\tau_0 + r(1-\tau_1)})(2\sqrt{r}-r)}.
\end{eqnarray}

\end{proof}

Observing Lemma \ref{lemma:converge2}, we can get that the first and second terms of this formula can be combined as long as we set $T \geq \mathcal{O}(N^2\epsilon^2/(d\log(1/\delta)))$. Meanwhile, combined with the privacy condition, the constraint on $T$ in Lemma \ref{T_lowerbound} can be transformed into a constraint on the dataset size $N$.

\begin{lemma}\label{N_bound}
When we set $T \geq \mathcal{O}(N^2\epsilon^2/(d\log(1/\delta)))$, $\eta = \sqrt{\frac{2B^2}{d\sigma^2T(L_0+L_1(\tau_0+1))}}$ and $\sigma = c_2 \frac{q\sqrt{T {\rm log}(1/\delta)}}{\epsilon}$, we only need to have a large enough $N>L_1A'(\epsilon, \delta, \tau, L, d, r)$, then the condition in Lemma 3 is satisfied.
\end{lemma}

\begin{proof}

Consider each condition in Lemma~\ref{lemma:eta_bound} in turn.
For the first two inequalities, Lemma~\ref{lemma:eta_bound} requires:
\begin{equation}
\eta = \sqrt{\frac{2B^2}{d\sigma^2T(L_0+L_1(\tau_0+1))}} \leq \min \left(\frac{(2\sqrt{r}-r)(1-\tau_1)\alpha}{4[L_0(1-\tau_1)+L_1(\sqrt{r}-r+\tau_0)]}, \frac{(1-\tau_1)\sqrt{r}\alpha}{4[L_0(1-\tau_1)+L_1\tau_0 + L_1\sqrt{r}]}\right)
\end{equation}
and
\begin{equation}
\eta = \frac{N\epsilon}{c_2T}\sqrt{\frac{2}{d(L_0+L_1(\tau_0+1))}} \leq \min \left(\frac{(2\sqrt{r}-r)(1-\tau_1)\alpha}{4[L_0(1-\tau_1)+L_1(\sqrt{r}-r+\tau_0)]}, \frac{(1-\tau_1)\sqrt{r}\alpha}{4[L_0(1-\tau_1)+L_1\tau_0 + L_1\sqrt{r}]}\right).
\end{equation}
The above conditions hold when we run our algorithm for enough iterations:
\begin{equation}
T \geq \max \left(\frac{4[L_0(1-\tau_1)+L_1(\sqrt{r}-r+\tau_0)]}{(2\sqrt{r}-r)(1-\tau_1)\alpha}, \frac{4[L_0(1-\tau_1)+L_1\tau_0 + L_1\sqrt{r}]}{(1-\tau_1)\sqrt{r}\alpha}\right)\frac{N\epsilon}{c_2}\sqrt{\frac{2}{d(L_0+L_1(\tau_0+1))}}.
\end{equation}
In fact, the above inequality holds naturally when we set $T \geq \mathcal{O}(N^2\epsilon^2/(d\log(1/\delta)))$.
Then, consider the last condition of Lemma \ref{lemma:eta_bound}:
\begin{equation}
\eta = \frac{N\epsilon}{c_2T}\sqrt{\frac{2}{d(L_0+L_1(\tau_0+1))}} \leq \frac{\alpha B^2}{6L_1d\sigma^2} = \frac{\alpha N^2 \epsilon^2}{6L_1dc_2^2T\log(1/\delta)},
\end{equation}
which is satisfied only if the number of samples $N$ in the dataset is large enough, so that the following formula holds:
\begin{equation}
N\geq L_1A'(\epsilon, \delta, \tau, L, d, r)=\frac{6L_1c_2}{\alpha \epsilon}\sqrt{\frac{2d\log(1/\delta)}{L_0+L_1(\tau_0+1)}}. 
\end{equation}
\end{proof}

~~~

~~~

~~~

\noindent\textbf{The proof of Corollary~1}
\begin{proof}
When $N>L_1A'(\epsilon, \delta, \tau, L, d, r)$, $T \geq \mathcal{O}(N^2\epsilon^2/(d\log(1/\delta)))$, it can be seen from Lemma \ref{N_bound} that the conditions in Lemma \ref{lemma:eta_bound} are established, so the conditions of Lemma \ref{lemma:converge2} are satisfied. Therefore, we can substitute $T \geq \mathcal{O}(N^2\epsilon^2/(d\log(1/\delta)))$ into Lemma \ref{lemma:converge2} to get the following result:

\begin{eqnarray}\nonumber
\mathbb{E}(\mathop{min}\limits_{0 \leq t < T}||\nabla f(x_t)||) \leq \mathcal{O}( \sqrt{\frac{ \sqrt{d\log(1/\delta)}}{N\epsilon}})+\frac{8\tau_0^2(1+\tau_0)}{3N(\tau_0, \tau_1, r)(1-\tau_1)^3(\tau_0 + \frac{r(1-\tau_1)}{2\tau_0 + r(1-\tau_1)})(2\sqrt{r}-r)}\\\nonumber
\end{eqnarray}

\end{proof}




~~

~~

~~

\section{A Simple Example of ``lazy region" on Logistic Regression}

We consider the same logistic regression setup as \citet{bu2022automatic} to demonstrate the ``lazy region" phenomenon. Specifically, we 
collect 10,000 positive samples from $\mathcal{N}(1,1)$ and 10,000 negative samples from $\mathcal{N}(-1,1)$. We train the model with SGD, DP-SGD, Auto-S, and DP-PSAC methods, respectively. We calculate the batch-averaged gradients $\sum clip(-y(1-\frac{1}{1+e^{-y(\theta+x)}}))$ using different methods under different settings of the logistic regression parameter $\theta$. The results are shown in Figure~\ref{figure:fig5}, where the clip threshold is set to 0.1 for DP-SGD and the hyperparameter $r$ of Auto-S and DP-PSAC is set to 0.01, which is the same as the setting of \citet{bu2022automatic}.

It is observed that DP-PSAC can maintain a large gradient size even $\theta$ is small, which means that DP-PSAC hardly gets stuck in a ``lazy region" situation. Meanwhile, DP-PSAC is closer to the original gradient than DP-SGD and Auto-S under all $\theta$.

\begin{figure}[h]
\centering
\includegraphics[width=0.5\columnwidth]{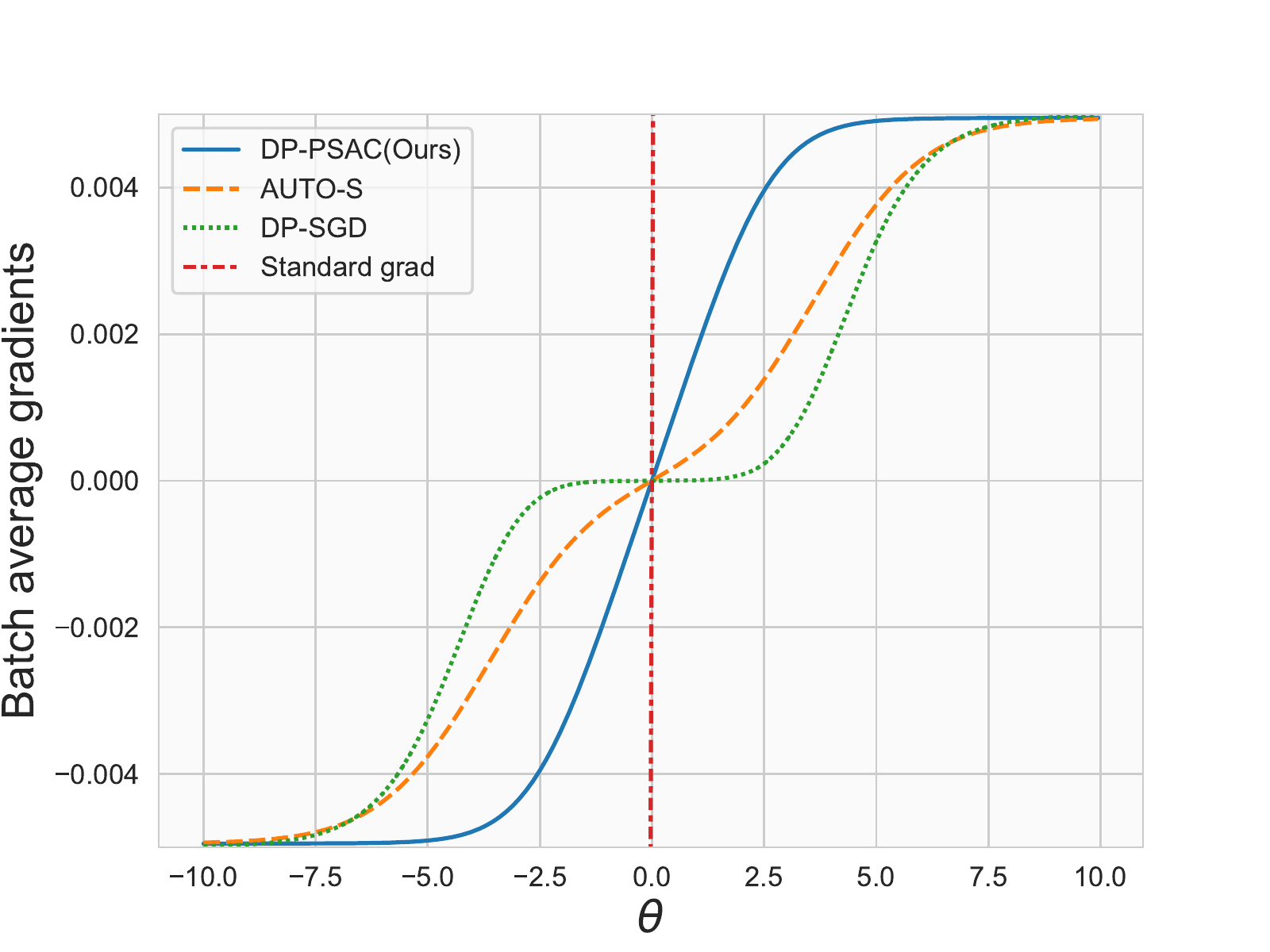} 
\caption{The batch-averaged gradients obtained by different methods 
for
the logistic regression parameter $\theta$.}
\label{figure:fig5}
\end{figure}

\section{More Experimental Results}
\subsection{Per-Sample Cosine Similarity on More Datasets}
\begin{figure}[h]
\centering
\includegraphics[width=0.4\columnwidth]{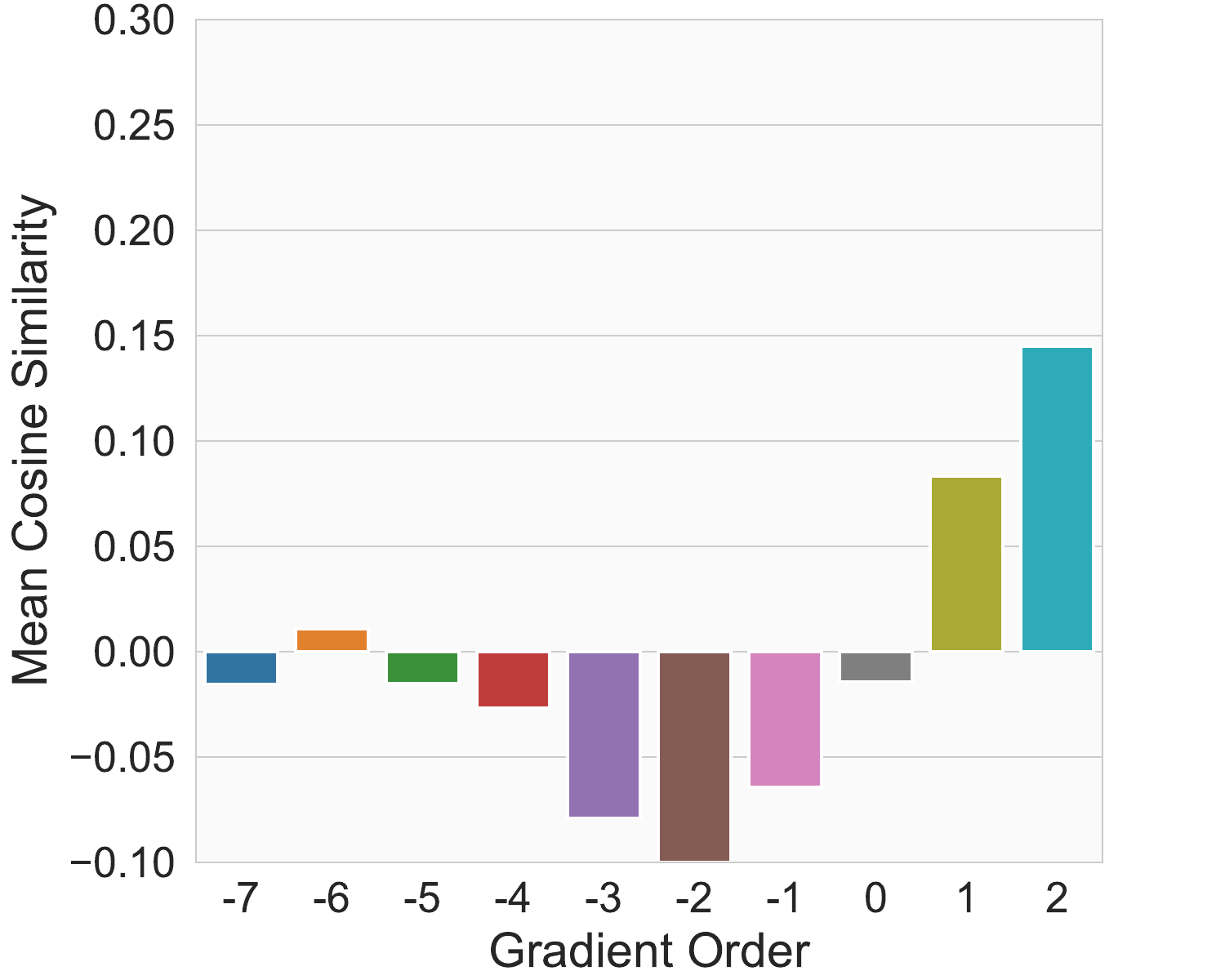} 
\includegraphics[width=0.4\columnwidth]{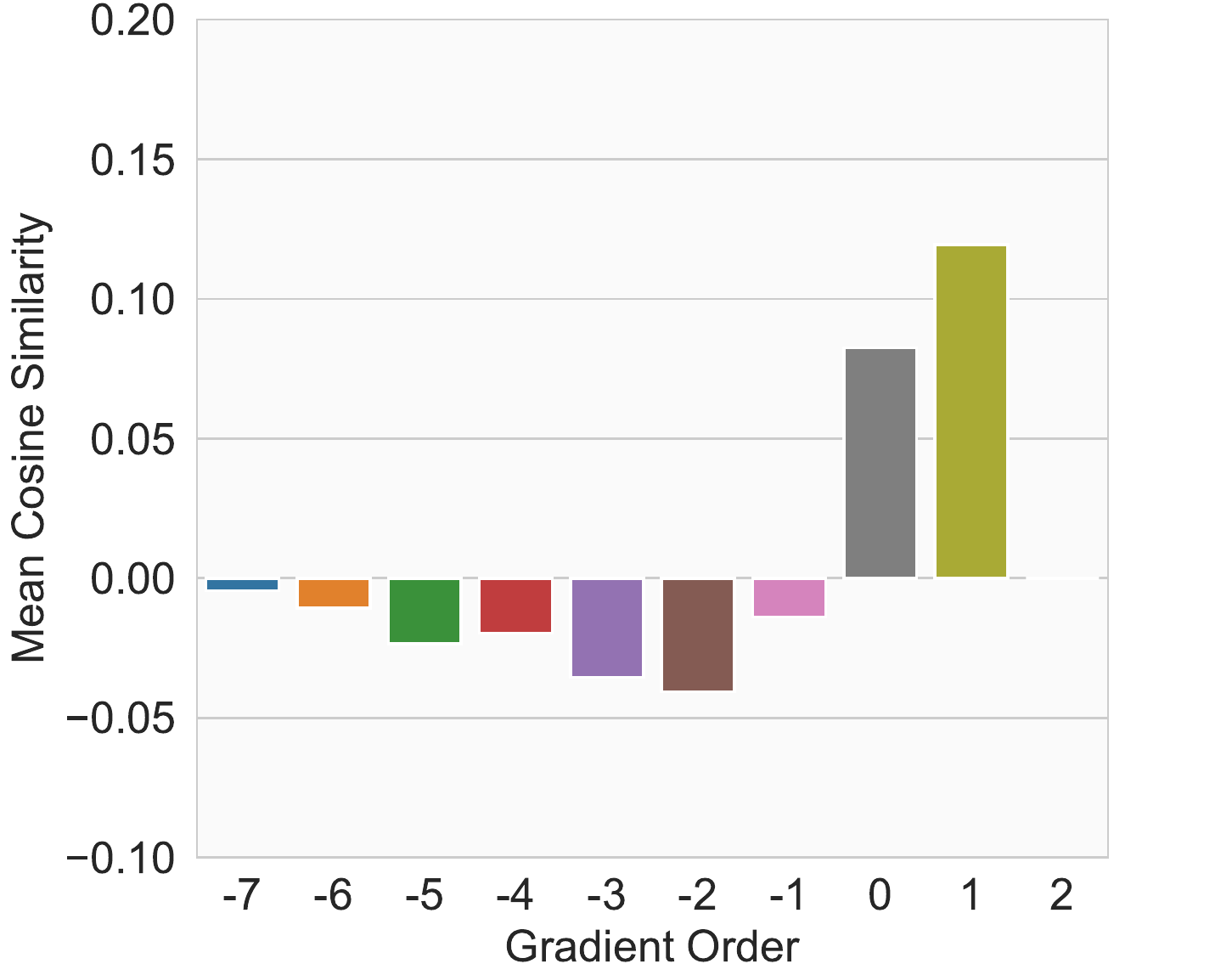} 
\caption{Average cosine similarity of single sample gradient and the batch-averaged gradient throughout training on FashionMNIST and CIFAR10 datasets with DP-SGD.}.
\label{figure:fig6}
\end{figure}

We also calculate the cosine similarity of each sample with its corresponding batch-averaged gradient on FashionMNIST and CIFAR10 datasets.  For FashionMNIST, we run DP-SGD under $(3,1e-5)$-DP with a simple CNN network, which is the same as the experimental setup on MNIST. While for CIFAR10, we use pre-trained SimCLRv2 to train a logistic regression model under $(2, 1e-5)$-DP, differentiated from MNIST and FashionMNIST, which is a convex optimization scenario. From Figure~\ref{figure:fig6} we observe similar statistical results as in MNIST under two different datasets and models. This helps to better demonstrate that small gradient samples contribute little to the batch-averaged gradient.





\end{document}